\documentclass[conference]{IEEEtran}
\IEEEoverridecommandlockouts
\usepackage{amsmath,amssymb,amsfonts}
\usepackage{graphicx}
\usepackage{textcomp}
\usepackage{xcolor}
\usepackage{amsthm}

\usepackage[utf8]{inputenc} % allow utf-8 input
\usepackage[T1]{fontenc}    % use 8-bit T1 fonts
\usepackage{url}            % simple URL typesetting
\usepackage{booktabs}       % professional-quality tables
\usepackage{amsfonts}       % blackboard math symbols
\usepackage{nicefrac}       % compact symbols for 1/2, etc.
\usepackage{microtype}      % microtypography
\usepackage{xcolor}         % colors
\usepackage{hyperref}       % hyperlinks
\usepackage{orcidlink}      % Support for ORCID icon
\usepackage{mathtools}      % includes and extends amsmath
\usepackage{graphicx}       % graphics handling
\usepackage{subcaption}     % subfigure captions (alternative: subfig)
\usepackage{wrapfig}        % allows wrapping text around figures
\usepackage[numbers,sort&compress]{natbib}  % Use [numbers] for numerical citations, or remove it for author-year
\usepackage{algorithm}      % algorithm floating environments
\usepackage{algpseudocode}  % algorithmic pseudocode styles
\usepackage{tcolorbox}      % colored boxes
\newtheorem{theorem}{Theorem}
\newtheorem{lemma}[theorem]{Lemma}
\newtheorem{corollary}[theorem]{Corollary}

\newtheorem{definition}[theorem]{Definition}
\usepackage{thmtools}
\usepackage{thm-restate}
\usepackage{colortbl}
\usepackage{enumitem}

\def\BibTeX{{\rm B\kern-.05em{\sc i\kern-.025em b}\kern-.08em
    T\kern-.1667em\lower.7ex\hbox{E}\kern-.125emX}}
\begin{document}

\title{Towards Strong Certified Defense with Universal Asymmetric Randomization
}

\author{
    \IEEEauthorblockN{Hanbin Hong}
    \IEEEauthorblockA{
        \textit{University of Connecticut}
    }
    \and
    \IEEEauthorblockN{Ashish Kundu}
    \IEEEauthorblockA{
        \textit{Cisco Research}
    }
    \and
    \IEEEauthorblockN{Ali Payani}
    \IEEEauthorblockA{
        \textit{Cisco Research}
    }
    \and
    \IEEEauthorblockN{Binghui Wang}
    \IEEEauthorblockA{
        \textit{Illinois Institute of Technology}
    }
    \and
    \IEEEauthorblockN{Yuan Hong}
    \IEEEauthorblockA{
        \textit{University of Connecticut}
    }
}

\maketitle

\begin{abstract}
 % Randomized smoothing has become essential for achieving certified adversarial robustness in machine learning models. However, current methods primarily use isotropic noise distributions that are uniform across all data dimensions, such as image pixels, limiting the effectiveness of robustness certification. To address this limitation, we propose UCAN: a novel technique that \underline{U}niversally amplifies randomized smoothing for \underline{C}ertified robustness with \underline{A}nisotropic \underline{N}oise. UCAN is designed to enhance any existing randomized smoothing method, converting it from symmetric (isotropic) to asymmetric (anisotropic) noise distributions, thereby offering a more tailored defense against adversarial attacks. Our theoretical framework is versatile, supporting a wide array of noise distributions for certified robustness in different $\ell_p$-norms. Besides the theoretical advancement, we develop a novel framework equipped with three exemplary noise parameter generators (NPGs) to customize the anisotropic noise, allowing for pursuing different levels of robustness enhancements in practical settings. Empirical evaluations underscore the significant leap in UCAN's performance over existing state-of-the-art methods, demonstrating as much as $182.6\%$ improvement of certified accuracy at large certified radii on MNIST, CIFAR10, and ImageNet datasets.

Randomized smoothing has become essential for achieving certified adversarial robustness in machine learning models. However, current methods primarily use isotropic noise distributions that are uniform across all data dimensions, such as image pixels, limiting the effectiveness of robustness certification by ignoring the heterogeneity of inputs and data dimensions. To address this limitation, we propose UCAN: a novel technique that \underline{U}niversally \underline{C}ertifies adversarial robustness with \underline{A}nisotropic \underline{N}oise. UCAN is designed to enhance any existing randomized smoothing method, transforming it from symmetric (isotropic) to asymmetric (anisotropic) noise distributions, thereby offering a more tailored defense against adversarial attacks. Our theoretical framework is versatile, supporting a wide array of noise distributions for certified robustness in different $\ell_p$-norms and applicable to any arbitrary classifier by guaranteeing the classifier's prediction over perturbed inputs with provable robustness bounds through tailored noise injection. Additionally, we develop a novel framework equipped with three exemplary noise parameter generators (NPGs) to optimally fine-tune the anisotropic noise parameters for different data dimensions, allowing for pursuing different levels of robustness enhancements in practice. %Furthermore, we introduce a generalized metric to accurately measure the certified robustness of both isotropic and anisotropic randomized smoothing, addressing the inadequacy of the traditional ``certified radius'' in the anisotropic context. 
Empirical evaluations underscore the significant leap in UCAN's performance over existing state-of-the-art methods, demonstrating up to $182.6\%$ improvement in certified accuracy at large certified radii on MNIST, CIFAR10, and ImageNet datasets.\footnote{Code is anonymously available at \href{https://github.com/youbin2014/UCAN/}{https://github.com/youbin2014/UCAN/}}
\end{abstract}

\begin{IEEEkeywords}
Certified Robustness, Adversarial Robustness\end{IEEEkeywords}
\section{Introduction}
\label{sec: Introduction}

Deep learning models have demonstrated remarkable performance across a wide range of applications. However, they are notoriously vulnerable to adversarial perturbations—carefully crafted minor modifications to inputs can lead to severe misclassification or misrecognition \citep{goodfellow2014explaining, carlini2017towards}. These adversarial examples pose significant threats in real-world scenarios, such as autonomous driving \cite{sun2020towards}, medical diagnosis \cite{ma2021understanding}, and face recognition systems \cite{dong2019efficient}, where even small prediction errors can result in catastrophic consequences.

To mitigate these vulnerabilities, robust defense mechanisms with certified guarantees are highly desirable. While empirical defense methods, including adversarial training \cite{madry2018towards, shafahi2019adversarial}, perturbation destruction \cite{xu2017feature, xie2018mitigating}, and feature regularization \cite{xie2019feature, yang2021adversarial}, have shown promise, they often fall short against adaptive and stronger adversaries \cite{athalye2018obfuscated, croce2020reliable, xie2019improving}. These methods typically cannot provide formal guarantees of robustness, as their defenses can be broken by more sophisticated attacks.

In contrast, certified robustness methods \cite{wong2018provable, cohen2019certified, lecuyer2019certified} offer provable guarantees by ensuring that no adversarial perturbation within a specified boundary—typically defined by an $\ell_p$-norm ball (e.g., $\ell_1$, $\ell_2$, or $\ell_\infty$)—can alter the model's prediction. Among these, randomized smoothing has emerged as a state-of-the-art (SOTA) technique due to its universal applicability to any arbitrary classifier \cite{lecuyer2019certified, teng2020ell, cohen2019certified}. By injecting noise to data during both the training and inference phases, randomized smoothing transforms any classifier into a smoothed classifier with certified robustness guarantees.

Despite its success, existing randomized smoothing methods predominantly rely on isotropic noise distributions, where the same noise parameters are uniformly applied across all data dimensions (e.g., all pixels in an image). This uniform approach overlooks the inherent heterogeneity of different data dimensions, potentially limiting the effectiveness of robustness certification. For instance, using identical noise for all pixels might not be optimal, as certain pixels may be more critical to the model's decision-making process than others. Consequently, the \emph{accuracy} of the smoothed classifier on both perturbed and clean inputs may be compromised, and the \emph{certified radius} based on the $\ell_p$-norm ball might not fully capture the true robustness potential. In reality, the $\ell_p$-norm ball serves as a \emph{sufficient but not necessary} condition for certification, as some regions outside the $\ell_p$-norm ball can still maintain robustness.

To address these limitations, we introduce UCAN: a novel technique that \underline{U}niversally \underline{C}ertifies adversarial robustness with \underline{A}nisotropic \underline{N}oise. UCAN significantly enhances any existing randomized smoothing method by transitioning from symmetric (isotropic) to asymmetric (anisotropic) noise distributions. This tailored noise injection allows for more effective and adaptive defenses against adversarial attacks by assigning different noise parameters to various data dimensions based on their specific characteristics, importance, and vulnerability.

Developing UCAN involves overcoming two primary and complex challenges:

\begin{enumerate}[label=\arabic*)]
    \item \textbf{Universal Certification Guarantee}: Establishing a robust theoretical foundation that universally supports various anisotropic noise distributions. This framework must ensure strict and sound certified robustness guarantees when different means and variances are assigned to different data dimensions.

\vspace{0.05in}
    
    \item \textbf{Optimal Noise Parameterization}: Designing mechanisms to derive appropriate means and variances for generating anisotropic noise tailored to various data dimensions. This ensures maximal certification performance without compromising robustness.
\end{enumerate}

To tackle these challenges, UCAN offers the following key contributions:

\begin{enumerate}[label=\arabic*)]
    \item \textbf{Universal Certification Theory for Anisotropic Noise}. We present a universal theory for certifying the robustness of randomized smoothing with any anisotropic noise distribution. This theory seamlessly transforms existing certifications based on isotropic noise into those with anisotropic noise, supporting various $\ell_p$-norm perturbations (e.g., $\ell_1$, $\ell_2$, $\ell_\infty$) and ensuring sound certified robustness across different noise distributions.

\vspace{0.05in}
    
    \item \textbf{Customizable Anisotropic Noise Parameter Generators (NPGs)}. We design three distinct NPGs, including two novel neural network-based methods, to efficiently generate the element-wise hyper-parameters (mean and variance) of the anisotropic noise distributions for all data dimensions (e.g., image pixels). These NPGs provide different efficiency-optimality trade-offs, significantly amplifying the certification performance while ensuring certified robustness.
    
   % \item \textbf{Auxiliary Metric for Certified Robustness}. We introduce the Alternative Lebesgue Measure (ALM), a novel auxiliary metric designed to complement the traditional ``certified radius'' in measuring certification performance under anisotropic noise. While the certified radius may not fully capture robustness in the anisotropic context, ALM provides a more comprehensive assessment by fully accounting for the certified region, offering a clearer and more accurate evaluation of certified robustness.

\vspace{0.05in}
    
    \item \textbf{Significantly Boosted Certification Performance}. Empirical evaluations on benchmark datasets (MNIST, CIFAR10, and ImageNet) demonstrate that UCAN drastically outperforms SOTA randomized smoothing-based certified robustness methods. Specifically, UCAN achieves up to $182.6\%$ improvement in certified accuracy at large certified radii compared to existing methods.
\end{enumerate}

In summary, UCAN represents a significant advancement in certified robustness by introducing anisotropic noise into randomized smoothing. This approach not only enhances theoretical robustness guarantees but also provides practical mechanisms for optimizing noise parameters, leading to substantial improvements in certification across various datasets.

The remainder of this paper is organized as follows. Section \ref{sec: Preliminary} introduces the preliminaries, including the threat model and the isotropic randomized smoothing method for anisotropic certified robustness. Section \ref{sec: ARS} presents the proposed universal theory and metrics for robustness region. Section \ref{sec: Customize} details the three different methods for customizing anisotropic noise to enhance certification performance. Section \ref{sec:soundness} discusses and proves the soundness of the certification-wise anisotropic noise. Section \ref{sec: Experiments} provides experimental results demonstrating UCAN's superior performance. Finally, Sections \ref{sec: RelatedWork} and \ref{sec: Conclusion} discuss related work and conclude the paper, respectively.

\section{Preliminaries}
\label{sec: Preliminary}
In this section, we provide a brief overview of randomized smoothing with isotropic noise for certified robustness, which forms the foundation for our proposed UCAN framework.

\subsection{Randomized Smoothing and Certified Robustness}

Randomized smoothing is a technique that constructs a smoothed classifier from an arbitrary base classifier by adding random noise to the inputs. The smoothed classifier inherits certain robustness properties, allowing for certified guarantees against adversarial perturbations within a specific norm ball.

Consider a classification task where inputs $x \in \mathbb{R}^d$ are mapped to labels in a finite set of classes $\mathcal{C}$. Given any base classifier $f: \mathbb{R}^d \to \mathcal{C}$, randomized smoothing defines a smoothed classifier $g$ that, for any input $x$, outputs the class most likely to be predicted by $f$ when noise is added to $x$. Specifically, let $\epsilon \in \mathbb{R}^d$ be noise drawn from an arbitrary isotropic probability distribution $\phi$ (e.g., $\mathcal{N}(0,1)$), and define the random variable $X = x + \epsilon$. The smoothed classifier $g$ is then defined as:

\vspace{-2mm}
\begin{equation}\label{eq:smoothed_classifier}
    g(x) = \arg\max_{c \in \mathcal{C}} \mathbb{P}(f(X) = c)
\end{equation}
\vspace{-8mm}

\subsection{Certified Robustness via Isotropic Randomized Smoothing}

Randomized smoothing provides a way to certify the robustness of the smoothed classifier $g$ against adversarial perturbations measured in terms of the $\ell_p$-norm. The core idea is that if the smoothed classifier predicts a class $c_A$ with high probability, and all other classes have significantly lower probabilities, then small perturbations to the input will not change the predicted class. We summarize existing results on certified robustness via randomized smoothing with isotropic noise in the following unified theorem.

\begin{theorem}[\textbf{Certified Robustness via Randomized Smoothing with Isotropic Noise}]
\label{thm:cohen's thm}
Let $f: \mathbb{R}^d \to \mathcal{C}$ be any (possibly randomized) base classifier, and let $\phi$ be an isotropic probability distribution used to generate noise $\epsilon \in \mathbb{R}^d$. Define the smoothed classifier $g$ as in Equation~\eqref{eq:smoothed_classifier}. For a specific input $x \in \mathbb{R}^d$, let $X = x + \epsilon$, and suppose that there exists a class $c_A \in \mathcal{C}$ and bounds $\underline{p_A}, \overline{p_B} \in [0,1]$ such that:

\vspace{-4mm}
\begin{equation}\label{eq:probability_bounds}
    \mathbb{P}(f(X) = c_A) \geq \underline{p_A} \geq \overline{p_B} \geq \max_{c \neq c_A} \mathbb{P}(f(X) = c)
\end{equation}
\vspace{-4mm}

Then, for any perturbation $\delta \in \mathbb{R}^d$ where $\|\delta\|_p < R(\underline{p_A}, \overline{p_B})$, the smoothed classifier $g$ will consistently predict class $c_A$ at $x + \delta$, i.e., $g(x + \delta) = c_A$. Here, $\|\cdot\|_p$ denotes the $\ell_p$-norm (with $p = 1, 2, \infty$), and $R(\underline{p_A}, \overline{p_B})$ is the certified radius, which depends on the noise distribution $\phi$ and the norm $\ell_p$.

\end{theorem}

The certified radius $R(\underline{p_A}, \overline{p_B})$ quantifies the robustness of the smoothed classifier $g$ around the input $x$. It ensures that any adversarial perturbation $\delta$ with $\|\delta\|_p < R(\underline{p_A}, \overline{p_B})$ cannot change the predicted class. The exact form of $R$ varies depending on the noise distribution $\phi$ and the norm $\ell_p$. For example, when $\phi$ is an isotropic Gaussian distribution with standard deviation $\sigma$, and the $\ell_2$ norm is considered, the certified radius is given by \citep{cohen2019certified}:

\begin{equation}\label{eq:gaussian_radius}
\small
    R = \frac{\sigma}{2} \left( \Phi^{-1}(\underline{p_A}) - \Phi^{-1}(\overline{p_B}) \right)
\end{equation}

where $\Phi^{-1}$ is the inverse cumulative distribution function (CDF) of the standard normal distribution.

\subsection{Limitations of Isotropic Noise in Randomized Smoothing}

While randomized smoothing with isotropic noise provides a powerful tool for certified robustness, it applies the same noise distribution uniformly across all data dimensions. This isotropic approach may not fully exploit the potential robustness, as it ignores the heterogeneity and varying importance of different input features. Certain dimensions (e.g., pixels in an image) may be more sensitive or critical to the classification task, and treating them uniformly can limit the defense performance.

To address these limitations, our work extends randomized smoothing to anisotropic noise distributions, where different noise parameters can be assigned to different data dimensions. This extension poses significant challenges in developing universal theoretical guarantees and in designing efficient methods to optimize the noise parameters. In the following sections, we present our proposed UCAN framework, which overcomes these challenges to enhance certified robustness.

\section{Certified Robustness with Anisotropic Noise}
\label{sec: ARS}

\begin{figure}
    \centering
    \includegraphics[width=\linewidth]{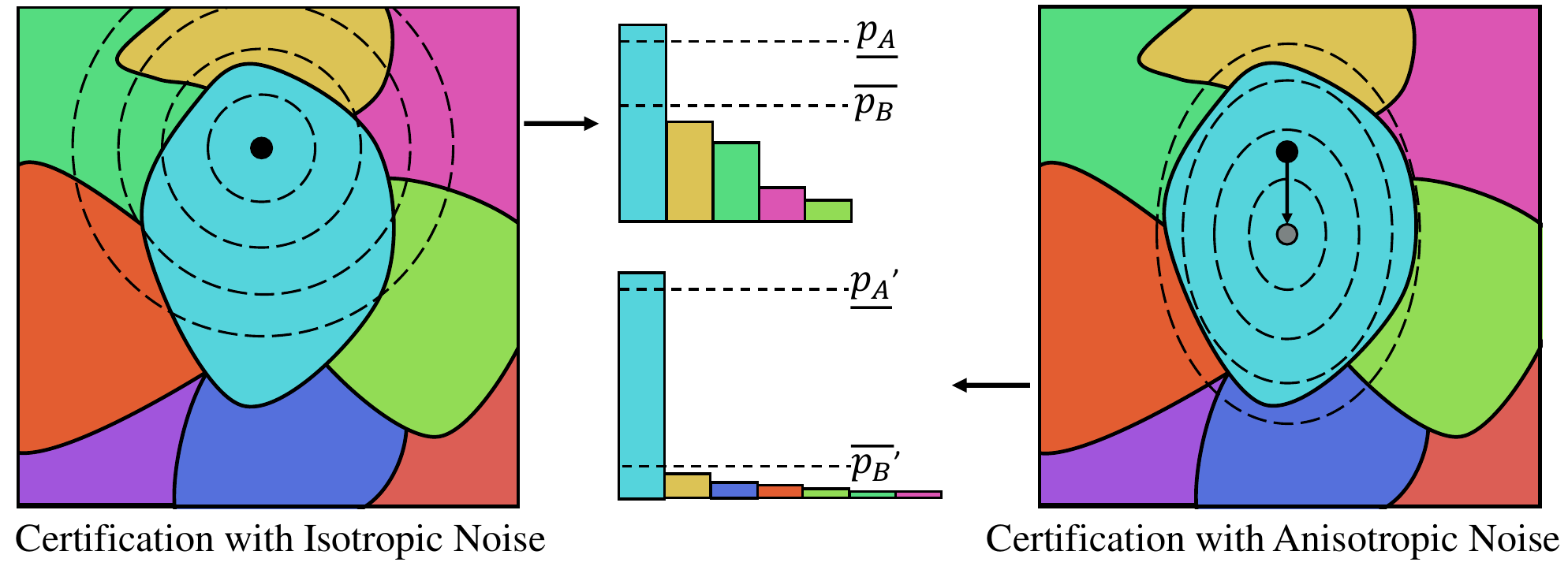}
    \caption{Anisotropic noise vs isotropic noise. The decision regions of $f$ are denoted in different colors. The dashed lines are the level sets of the noise distribution. The left figure shows the RS with isotropic Gaussian noise $\mathcal{N}(0, \lambda^2 \mathbf{I})$ in \cite{cohen2019certified} whereas the right figure shows the RS with anisotropic Gaussian noise $\mathcal{N}(\mu, \Sigma)$, where $\Sigma=\lambda^2 diag (\sigma_1^2, \sigma_2^2, ..., \sigma_d^2)$. Certification can be improved by enlarging the gap between $\underline{p_A}$ and $\overline{p_B}$.}\vspace{-0.2in}
    \label{fig:overview}
\end{figure}

In this section, we establish a universal theory for the certification via randomized smoothing with \emph{anisotropic} noise. Given any isotropic randomized smoothing methods, our method can universally transform them to anisotropic randomized smoothing for enhanced certified robustness. 

\subsection{General Anisotropic Noise}

Given an arbitrary isotropic noise $\epsilon$, we define the corresponding anisotropic noise $\epsilon'$ as:
\begin{equation}
\label{eq:anisotropic_noise_general}
\epsilon' = \epsilon^\top \Sigma + \mu
\end{equation}
where $\Sigma \in \mathbb{R}^{d \times d}$ is an \emph{invertible} covariance matrix, and $\mu \in \mathbb{R}^d$ is the mean offset vector. The covariance matrix $\Sigma$ introduces dependencies between different dimensions of the noise, capturing potential correlations, while $\mu$ allows for mean shifts in the noise distribution. See Figure \ref{fig:visualization} for examples.

Then, certified robustness with anisotropic noise can be ensured per Theorem \ref{thm:our thm}.

\begin{restatable}[\textbf{Asymmetric Randomized Smoothing via Universal Transformation}]{theorem}{thmone}
\label{thm:our thm}
Let $f: \mathbb{R}^d \to \mathcal{C}$ be any deterministic or randomized function. Suppose that for the multivariate random variable with isotropic noise $X = x + \epsilon$ in Theorem~\ref{thm:cohen's thm}, the certified radius function is $R(\cdot)$. Then, for the corresponding anisotropic input $Y = x + \epsilon^\top \Sigma + \mu$, if there exist $c'_A \in \mathcal{C}$ and $\underline{p_A}', \overline{p_B}' \in [0,1]$ such that:
\begin{equation}
\label{eq:condition_general}
\mathbb{P}\left( f(Y) = c'_A \right) \geq \underline{p_A}' \geq \overline{p_B}' \geq \max_{c \neq c'_A} \mathbb{P}\left( f(Y) = c \right)
\end{equation}
then for the anisotropic smoothed classifier $g'(x + \delta') \equiv \arg \max_{c \in \mathcal{C}} \mathbb{P}\left( f(Y + \delta') = c \right)$, we can guarantee $g'(x + \delta') = c'_A$ for all perturbations $\delta' \in \mathbb{R}^d$ such that:
\begin{equation}
\label{eq:guarantee_general}
\left\| \Sigma^{-1} \delta' \right\|_p \leq R\left( \underline{p_A}', \overline{p_B}' \right)
\end{equation}
provided that $\Sigma$ is invertible.
\end{restatable}

\begin{proof}
See the detailed proof in Appendix \ref{apd:proof}. 
\end{proof}

The general anisotropic noise with covariance allows for capturing correlations between different dimensions of the input data. However, adopting this generalized anisotropic noise introduces certain practical limitations: 1) \textbf{Invertibility of $\Sigma$}:
Our theory requires the covariance matrix $\Sigma$ to be invertible. If $\Sigma$ is not invertible, a small regularization term can be added to its diagonal to ensure invertibility and retain the validity of our certification guarantees. 2) \textbf{Computational Complexity}: Optimizing or learning a full covariance matrix $\Sigma \in \mathbb{R}^{d \times d}$ is computationally intensive, especially for high-dimensional data (e.g., images with $d = 150{,}528$ for ImageNet). The memory and computational requirements scale quadratically with the input dimension, making it impractical for large-scale applications.

Given these limitations, in practice, one might consider structured covariance matrices that balance expressiveness with computational efficiency, such as low-rank approximations, block-diagonal matrices, or sparse covariance matrices. In this paper, we focus on the independent anisotropic noise where noise parameters are independent along dimensions.

\begin{table*}[!t]
    \centering
    \caption{Certified radii (binary-case) for randomized smoothing with independent isotropic and anisotropic noise. $d$ is the dimension size. $\Phi^{-1}$ is the inverse CDF of Gaussian distribution. $\lambda$ is the scalar parameter of the isotropic noise. $\sigma$ is the anisotropic scale multiplier.}
    \resizebox{\linewidth}{!}{
    \begin{tabular}{c c c c c c c}
    \hline
    \hline
         Distribution & PDF & Adv. & Isotropic Guarantee& $\ell_p$ Radius for Iso. RS &Anisotropic Guarantee &   $\ell_p$ Radius for Ani. RS \\
         % & Alt. Lebesgue Measure \\
    \hline
         Gaussian \cite{cohen2019certified} & $\propto e^{-\|\frac{z}{\lambda}\|_2^2}$ & $\ell_2$ & $ ||\delta||_2 \leq\lambda (\Phi^{-1}(\underline{p_A}))$ & $\lambda (\Phi^{-1}(\underline{p_A}))$ & $||\delta \oslash \sigma||_2 \leq\lambda (\Phi^{-1}(\underline{p_A}'))$ & $\min \{\sigma\}\lambda (\Phi^{-1}(\underline{p_A}))$ \\
         % &  $ \sqrt[d]{\prod_{i=1}^d\sigma_i} \lambda (\Phi^{-1}(\underline{p_A}'))$\\
         Gaussian \cite{yang2020randomized} & $\propto e^{-\|{\frac{z}{\lambda}}\|_2^2}$ & $\ell_1$ &  $||\delta||_1 \leq\lambda (\Phi^{-1}(\underline{p_A}))$ & $\lambda (\Phi^{-1}(\underline{p_A}))$&  $||\delta \oslash \sigma||_1 \leq \lambda (\Phi^{-1}(\underline{p_A}'))$ &$\min \{\sigma\} \lambda (\Phi^{-1}(\underline{p_A}))$ \\
         % & $\sqrt[d]{\prod_{i=1}^d\sigma_i} \lambda (\Phi^{-1}(\underline{p_A}'))$\\
         &  & $\ell_\infty$ & $ ||\delta||_\infty \leq\lambda (\Phi^{-1}(\underline{p_A}))/\sqrt{d}$ & $\lambda (\Phi^{-1}(\underline{p_A}))/\sqrt{d}$ & $ ||\delta \oslash \sigma||_\infty \leq \lambda (\Phi^{-1}(\underline{p_A}'))/\sqrt{d}$ &$\min \{\sigma\} \lambda (\Phi^{-1}(\underline{p_A}))/\sqrt{d}$ \\
         % & $\sqrt[d]{\prod_{i=1}^d\sigma_i} \lambda (\Phi^{-1}(\underline{p_A}'))/\sqrt{d}$ \\
         Laplace \cite{teng2020ell} & $\propto e^{-\|\frac{z}{\lambda}\|_1}$ & $\ell_1$ & $||\delta||_1 \leq -\lambda \log(2(1-\underline{p_A}))$ & $-\lambda \log(2(1-\underline{p_A}))$& $||\delta \oslash \sigma||_1 \leq- \lambda \log(2(1-\underline{p_A}'))$ &$-\min \{\sigma\}\lambda \log(2(1-\underline{p_A}))$ \\
         % & $\sqrt[d]{\prod_{i=1}^d\sigma_i} \lambda \log(2(1-\underline{p_A}'))$\\ 
         Exp.  $\ell_\infty$ \cite{yang2020randomized} & $\propto e^{-\|\frac{z}{\lambda}\|_\infty}$ & $\ell_1$ & $ ||\delta||_1 \leq 2d\lambda(\underline{p_A}-\frac{1}{2})$ & $2d\lambda(\underline{p_A}-\frac{1}{2})$& $||\delta \oslash \sigma||_1 \leq 2d \lambda(\underline{p_A}'-\frac{1}{2})$ &$2\min \{\sigma\}d\lambda(\underline{p_A}-\frac{1}{2})$ \\
         % & $2\sqrt[d]{\prod_{i=1}^d\sigma_i}d\lambda(\underline{p_A}'-\frac{1}{2})$\\
          &  & $\ell_\infty$ & $||\delta||_\infty \leq \lambda \log(\frac{1}{2(1-\underline{p_A})})$ & $\lambda \log(\frac{1}{2(1-\underline{p_A})})$& $||\delta \oslash \sigma||_\infty \leq \lambda \log(\frac{1}{2(1-\underline{p_A}')})$ &$\min \{\sigma\}\lambda \log(\frac{1}{2(1-\underline{p_A})})$ \\ 
          % & $ \sqrt[d]{\prod_{i=1}^d\sigma_i}\lambda \log(\frac{1}{2(1-\underline{p_A}')})$\\
         Uniform $\ell_\infty$ \cite{lee2018simple}& $\propto \mathbb{I}(\|z\|_\infty\leq\lambda)$ & $\ell_1$ & $||\delta||_1 \leq 2\lambda (\underline{p_A}-\frac{1}{2})$ & $2\lambda (\underline{p_A}-\frac{1}{2})$ & $ ||\delta \oslash \sigma||_1 \leq 2 \lambda (\underline{p_A}'-\frac{1}{2})$ &$2\min \{\sigma\}\lambda (\underline{p_A}-\frac{1}{2})$ \\
         % & $2 \sqrt[d]{\prod_{i=1}^d\sigma_i}\lambda  (\underline{p_A}'-\frac{1}{2})$\\
         & & $\ell_\infty$ & $||\delta||_\infty \leq 2\lambda(1-\sqrt[d]{\frac{3}{2}-\underline{p_A}})$ & $2\lambda(1-\sqrt[d]{\frac{3}{2}-\underline{p_A}})$& $||\delta \oslash \sigma||_\infty \leq 2\lambda(1-\sqrt[d]{\frac{3}{2}-\underline{p_A}'})$ &$2\min \{\sigma\}\lambda(1-\sqrt[d]{\frac{3}{2}-\underline{p_A}})$ \\
         % & $2 \sqrt[d]{\prod_{i=1}^d\sigma_i}\lambda(1-\sqrt[d]{\frac{3}{2}-\underline{p_A}'})$\\
         Power Law $\ell_\infty$ \cite{yang2020randomized} & $\propto \frac{1}{(1+\|\frac{z}{\lambda}\|_\infty)^a}$ & $\ell_1$ & $ ||\delta||_1 \leq \frac{2d\lambda}{a-d}(\underline{p_A}-\frac{1}{2})$ & $\frac{2d\lambda}{a-d}(\underline{p_A}-\frac{1}{2})$ & $||\delta \oslash \sigma||_1 \leq \frac{2d\lambda}{a-d}(\underline{p_A}-\frac{1}{2})$ &$\min \{\sigma\}\frac{2d\lambda}{a-d}(\underline{p_A}-\frac{1}{2})$ \\
         % & $2\sqrt[d]{\prod_{i=1}^d\sigma_i}\frac{d\lambda}{a-d}(\underline{p_A}-\frac{1}{2})$\\
    \hline
    \hline
    \end{tabular}}
    \vspace{-3mm}
     \label{tab:transform}
\end{table*}

\vspace{+0.05in}
\noindent \textbf{Special Case: Diagonal Covariance Matrix}. The case where $\Sigma$ is a diagonal matrix corresponds to anisotropic noise with independent dimensions (i.e., no covariance between dimensions). Let $\Sigma = \operatorname{diag}(\sigma_1, \sigma_2, \ldots, \sigma_d)$, where $\sigma_i > 0$ for all $i$. In this case, $\Sigma^{-1} = \operatorname{diag}\left( \frac{1}{\sigma_1}, \frac{1}{\sigma_2}, \ldots, \frac{1}{\sigma_d} \right)$.

\begin{corollary}[\textbf{Asymmetric Randomized Smoothing with Independent Anisotropic Noise}]
\label{corollary:diagonal}
Under the same conditions as Theorem~\ref{thm:our thm} if $\Sigma = \operatorname{diag}(\sigma_1, \sigma_2, \ldots, \sigma_d)$, then the certified robustness guarantee becomes:
\begin{equation}
\label{eq:guarantee_diagonal}
\left\| \delta' \oslash \sigma \right\|_p \leq R\left( \underline{p_A}', \overline{p_B}' \right)
\end{equation}
where $\sigma = [\sigma_1, \sigma_2, \ldots, \sigma_d]^\top$, $\oslash$ denotes element-wise division, and $\delta' \in \mathbb{R}^d$ is the perturbation in the anisotropic space.
\end{corollary}

\begin{proof}
Since $\Sigma$ is diagonal, its inverse is also diagonal with entries $1/\sigma_i$. Therefore, we have:
\begin{equation}
\left\| \Sigma^{-1} \delta' \right\|_p = \left\| \delta' \oslash \sigma \right\|_p
\end{equation}
Substituting this into Equation~\eqref{eq:guarantee_general} yields the desired result.
\end{proof}

Theorem \ref{thm:our thm} provides a robustness guarantee with anisotropic noise, building on traditional isotropic RS theories. Equation \eqref{eq:guarantee_general} from this theorem is both explicit and widely applicable, significantly enhancing existing RS frameworks' performance. In Table \ref{tab:transform}, we outline some representative isotropic RS methods and their extension to anisotropic noise via Theorem \ref{thm:our thm}. Our method can also be readily adapted to other RS methods like those in \cite{zhang2020black,hong2022unicr} that lack explicit radius certifications, applying it directly to their numerical results. Details on the binary classifier version are discussed in Appendix \ref{apd:binary}.

In Figure \ref{fig:overview}, we clearly illustrate the significant benefits that the anisotropic noise can bring to randomized smoothing. In Theorem \ref{thm:our thm}, we observe that the mean offset $\mu$ does not affect the derivation of the certified robustness with anisotropic noise (Equation \eqref{eq:guarantee_general}). Thus, it is likely that the gap between probabilities $\underline{p_A}'$ and $\overline{p_B}'$ can be improved by a properly chosen mean offset of the anisotropic noise (without affecting the robustness guarantee). Additionally, with the heterogeneous variance, the anisotropic noise can better fit the different dimensions of the input without causing over-distortion.

\begin{corollary}
\label{corollary}
For the anisotropic input $Y$ in Theorem \ref{thm:our thm}, if Equation (\ref{eq:condition_general}) is satisfied, then $g'(x+\delta') \equiv \arg \max_{c\in \mathcal{C}} \mathbb{P}(f(Y+\delta')=c) =c'_A $ for all $||\delta'||_p\leq R'$ such that
\begin{equation}
\label{eq: worse-case guarantee}
    R'= \min\{\sigma\}R
\end{equation}
where $R$ is the certified radius of randomized smoothing via isotropic noise, and $\min \{\cdot \}$ denotes the minimum entry.

\begin{proof}
    The guarantee in Theorem \ref{thm:our thm} holds for $||\delta' \oslash \sigma||_p  \leq R$. Since $||\delta' \oslash \sigma||_p \leq ||\frac{\delta'}{\min \{\sigma\}}||_p$, if $||\frac{\delta'}{\min \{\sigma\}}||_p \leq R$, the guarantee still holds. This requires $||\delta'||_p\leq \min \{\sigma\} R$.
\end{proof}
\end{corollary}

\subsection{Certified Region and Two Metrics}

\begin{figure}
    \centering
    \includegraphics[width=\linewidth]{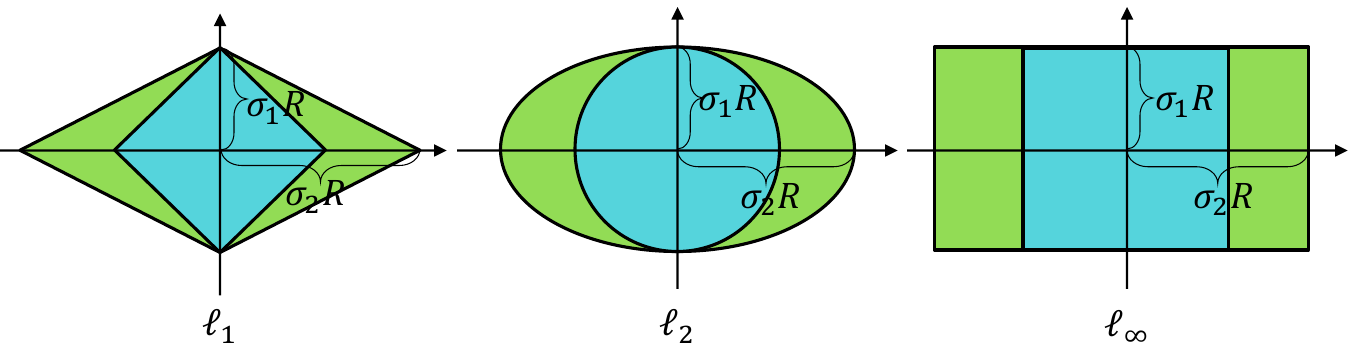}\
    \caption{An example illustration of the full certified region (green) and the certified radius (blue), where $\sigma_1 < \sigma_2$.}\vspace{-0.2in}
    \label{fig:anisotropic_vs_isotropic}
\end{figure}

In Corollary \ref{corollary}, we derive the certified radii $R'$ for asymmetric randomized smoothing in the formation of $\ell_p$-ball, i.e., $||\delta'||_p\leq R'$. While the certified radius reflects the maximum size of the tolerated perturbation within the $\ell_p$-ball. However, especially under asymmetric circumstances, the certified region can be highly asymmetric, and the $\ell_p$-ball (which is symmetric) may only represent a subset of the full robustness region, which is given by $(\sum_i^d(\frac{\delta'_i}{\sigma_i})^p)^\frac{1}{p} \leq R(\underline{p_A}', \overline{p_B}')$, as illustrated in Figure \ref{fig:anisotropic_vs_isotropic}. Therefore, to provide a more comprehensive and complementary assessment, we adopt an additional metric to measure the overall size of the certified region.

\vspace{0.05in}

\noindent\textbf{Radius and ALM for Certified Region}. In addition to the $\ell_p$ radius, we also adopt the Alternative Lebesgue Measure (ALM\footnote{This metric is mathematically equivalent to the ``proxy radius'' in \cite{eiras2021ancer}.}) for measuring the certified region. Specifically, we consider the certified region under the anisotropic guarantee as a $d$-dimensional super-ellipsoid, as defined in Definition \ref{def:super-ellipsoid delta}. It is worth noting that the $\ell_p$-norm ball is a sufficient but not necessary condition for certified robustness: while robustness within the $\ell_p$ ball is guaranteed, certain points outside this ball may also remain robust due to the true shape of the certified region. The super-ellipsoid formulation can represent this broader space.

\begin{definition}[\textbf{d-dimensional Generalized Super-ellipsoid of $\delta$}]
\label{def:super-ellipsoid delta}
The d-dimensional generalized super-ellipsoid ball of $\delta$ is defined as
\vspace{-8pt}
\begin{equation}
\scriptsize
    S(d,p)=\{(\delta_1,\delta_2,...,\delta_d):\sum_{i=1}^d|\frac{\delta_i}{\sigma_iR}|^{p}\leq 1, p>0\}
\end{equation}
\vspace{-20pt}
\end{definition}

\begin{theorem}[\textbf{Lebesgue Measure of the Robust Perturbation Set $S(d,p)$}]
\label{thm: Lebesgue measure}
Define $S(d,p)$ per Definition \ref{def:super-ellipsoid delta}, then the Lebesgue measure of the robust perturbation set is given by
\begin{equation}
\scriptsize
    V_S(d,p)=\frac{(2R\Gamma(1+\frac{1}{p}))^d\prod_{i=1}^d \sigma_i}{\Gamma(1+\frac{d}{p})}
\end{equation}
where $\Gamma$ is the Euler gamma function defined in Definition \ref{def: euler gamma} in Appendix \ref{apd: metric}.
\begin{proof}
    See the detail proof in Appendix \ref{apd: proof thm 5.4}.
\end{proof}
\end{theorem}

Recall that we aim to find an auxiliary metric that can measure the volume of this super-ellipsoid, so we derive the Lebesgue Measure of this super-ellipsoid as in Theorem \ref{thm: Lebesgue measure} since the Lebesgue Measure quantifies the ``volume'' of high-dimensional space, and then we simplify it by removing the constants w.r.t. the dimension and $p$, as well as transforming it to the radius scale. To this end, the ALM can be simplified as: $ALM=\sqrt[d]{\prod_{i=1}^d\sigma_i}R$.

\begin{figure*}[!th]
\centering
    \begin{subfigure}{0.4\linewidth}
    % \begin{minipage}[]{\linewidth}
    \centering
    \includegraphics[width=0.9\linewidth]{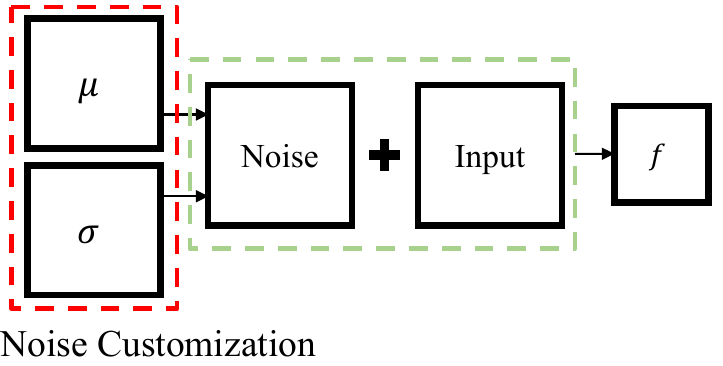}
    % \label{fig:unified framework}
    % \end{minipage}
    \end{subfigure}
% \hspace{-0.1in}
\begin{subfigure}{0.55\linewidth}
    % \begin{minipage}[]{0.55\linewidth}
    \centering
    \includegraphics[width=0.98\linewidth]{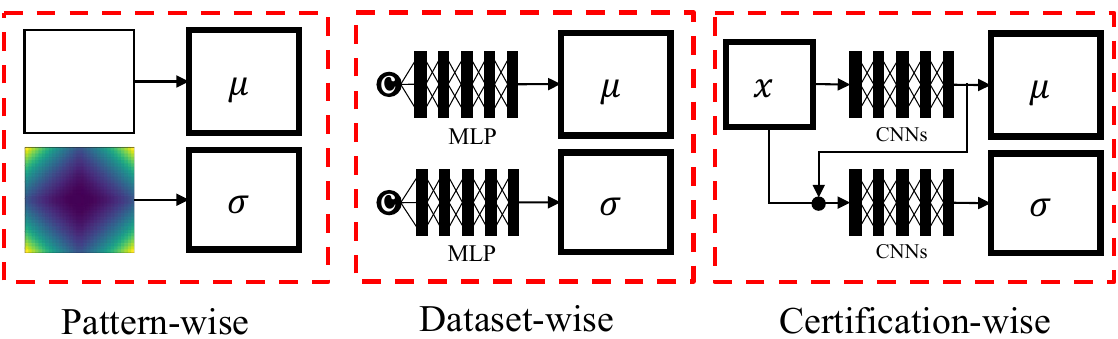}
    % \label{fig:three methods}
    % \end{minipage}
    \end{subfigure}
\caption{Left: the framework for customizing anisotropic noise. Right: three example noise parameter generators (NPGs).}\vspace{-0.15in}
\label{fig:unif}
\end{figure*}

It serves as an additional measure for evaluating the robustness region. This is particularly useful because the $\ell_p$ radius (\emph{as a more strict metric}) cannot fully capture the robustness region in certain dimensions due to its inherent symmetry, as illustrated in Figure \ref{fig:anisotropic_vs_isotropic}. In summary, the $\ell_p$ radius serves as a conservative, worst-direction certificate—being controlled by the smallest anisotropic scale $\min_i \sigma_i$—whereas the ALM provides an auxiliary, volume-oriented characterization of the full certified region through the geometric mean $\big(\prod_{i=1}^d \sigma_i\big)^{1/d}$. It is worth noting that ALM subsumes the certified radius: in any dimension, the certified radius corresponds to the smallest semi-axis length of the certified region, while ALM (as the geometric mean of all semi-axes times $R$) is always greater than or equal to this value, and coincides with it in the isotropic case. Therefore, ALM serves as an auxiliary metric for the size of certified region instead of a formal guarantee. See Appendix \ref{apd: metric} for detailed analysis and discussions for ALM.

\section{Customizing Anisotropic Noise}
\label{sec: Customize}
% We have established the theories to universally bound the perturbation and evaluate the certification performance of randomized smoothing with anisotropic noise. 

Theorem \ref{thm:our thm} formally guarantees robustness when heterogeneous noise parameters are assigned across different data dimensions. However, finding more optimal heterogeneous noise parameters rather than randomly assigning them remains a challenge. To this end, in UCAN, we design a unified framework to customize anisotropic noise for randomized smoothing (left in Figure \ref{fig:unif}), which includes three \emph{noise parameter generators} (NPGs) with different scales of trainable parameters (right in Figure \ref{fig:unif}) and optimality levels. 

Here, the ``optimality'' refers to the degree to which each NPG can optimize the noise parameters to maximize certified robustness (measured by either certified radius or ALM), while maintaining prediction accuracy. %The level of optimality reflects both the granularity and adaptivity with which noise parameters are determined:

\begin{itemize}[leftmargin=20pt]
% \setlength{\itemsep}{1pt}
%\vspace{-2mm}

\item \textbf{Pattern-wise Anisotropic Noise} (Low optimality): Uses pre-defined patterns for noise variances, offering basic but non-adaptive and relatively lower optimal robustness.  

\vspace{0.05in}

\item \textbf{Dataset-wise Anisotropic Noise} (Moderate optimality): Learns a global set of noise parameters optimized for the entire dataset, enabling some adaptation for different input data but still derived at the dataset level.

\vspace{0.05in}

\item \textbf{Certification-wise Anisotropic Noise} (High optimality): Generates noise parameters specifically tailored to each individual input at the certification time, achieving the most fine-grained and input-specific robustness optimization.

\end{itemize}

The optimality levels reflect a fundamental trade-off: higher-optimality approaches allow more precise and effective maximization of certified robustness, but require increased computational resources for training and/or inference. These three NPGs are representative examples—other designs are possible depending on application needs. In any case, the covariance matrix should be kept invertible (e.g., by adding a small positive constant to the diagonal if necessary). Implementation details and specific algorithms can be found in Appendix \ref{apd: algorithms}.

\subsection{Pattern-wise Anisotropic Noise}

The motivation for pattern-wise anisotropic noise in NPG is based on the understanding that different data regions affect predictions differently \citep{gilpin2018explaining}. Typically, an image's center contains more critical visual information, requiring lower variance to preserve clarity, whereas the borders may accommodate higher variance without significantly impacting predictions. Thus, this NPG utilizes fixed spatial patterns for anisotropic noise.

\begin{figure}
    \centering
    \includegraphics[width=0.95\linewidth]{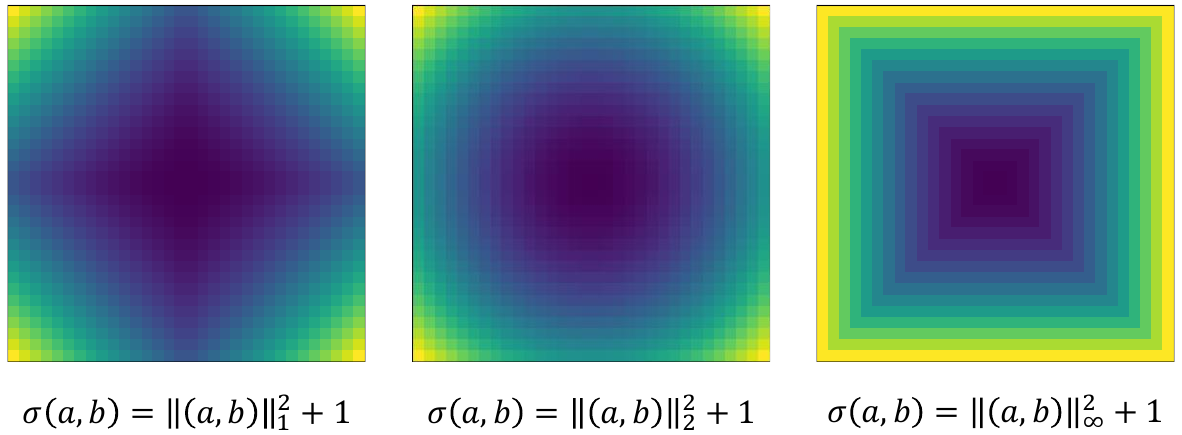}
    \caption{Spatial distributions for noise variances  (pattern-wise).}\vspace{-0.15in}
    \label{fig:pre-assigned vis}
\end{figure}

% Specifically, let the spatial distribution of the variances follow a function $\sigma(a,b)$, where $a$ and $b$ are the pixel's coordinates of the horizon and vertical axes such that the center of the image is denoted as $(a,b)=(0,0)$. Recall that the central variance can be intuitively set to be smaller than the border's variance. Thus, we design three different types of spatial distributions:

Consider a function $\sigma(a,b)$ representing the variance distribution across an image, where $(a,b)$ are the coordinates with the center at $(0,0)$. The variance is intuitively smaller at the center than at the borders due to varying feature importance. We propose three distinct spatial distribution types:

\begin{equation}
\label{eq:patterns}
\small
    \sigma(a,b)=\kappa||(a,b)||_p^2 +\iota, \ \ p=1,2,\infty
\end{equation}

where $||(a,b)||_p$ denotes the $\ell_p$-norm distance between $(a,b)$ and $(0,0)$, $\kappa$ is a constant parameter tuning the overall magnitude of the variance, $\iota$ denotes the variance of a center pixel since $\sigma(0,0)=\iota$, and $\iota>0$ such that $ \sigma(a,b) > 0$. Figure \ref{fig:pre-assigned vis} shows three example spatial distributions when $p=1,2,\infty$, where $\kappa=1$ and $\iota=1$. The noise mean is set as $0$ to avoid unnecessary deviation in the images.

% The pattern-wise anisotropic noise is an intuitive modification by considering the different contributions to the prediction by different portions of the data (e.g., pixels). However, it is not sufficiently fine-tuned and may hinder the performance if not well-crafted, especially considering the diverse characteristics of different datasets. 

Pattern-wise anisotropic noise modifies predictions based on spatial contributions (e.g., pixels). Yet, without fine-tuning, this approach may impair performance, particularly with datasets' diverse characteristics.
% For example, the MNIST dataset \cite{lecun2010mnist} may need an anisotropic noise with low variance in the center since the digital number tends to present in the center (see Section \ref{apd: visualization} for the visualized examples); a surveillance dataset \cite{oh2011large,baltieri20113dpes} or sports dataset \cite{chen2020rit,zhang2017martial} may need special patterns for the anisotropic variances since the humans or the athletes may occur in some specific positions according to the physical scenes. 
% Therefore, we may design pattern-fixed anisotropic noise in terms of the application scenario of specific datasets, though this process relies on the domain knowledge of the dataset and may need significant efforts to construct the patterns for different datasets. 
% Next, we propose an automated approach that can fine-tune the spatial distribution of the variances for the anisotropic noise for each dataset. 
Then we propose an automated method to fine-tune the variance spatial distribution of anisotropic noise for each dataset.

% \textbf{Examples}. 

% \begin{figure}
%     \centering
%     \includegraphics[width=\linewidth]{fig/preassigned example.pdf}
%     \caption{Examples of parameter generation for pattern-fixed anisotropic noise}
%     \label{fig:universal examples}
% \end{figure}

\subsection{Dataset-wise Anisotropic Noise}

% This NPG adopts a fix-input neural network generator \citep{creswell2018generative} to learn anisotropic variances during the robust training (i.e., training with noise), see Figure \ref{fig:unif} (right). Specifically, the generator takes a fixed constant as the input and outputs an anisotropic $\sigma$ tensor and $\mu$ tensor with each element $\mu_i$ and $\sigma_i$ denoting the mean offset and the multiplier of noise (viz. variance) per pixel, respectively. Given any noise $\epsilon$ drawn from any distribution, the anisotropic noise can be generated as $\epsilon'=\epsilon^\top\sigma+\mu$, and added to the input for randomized smoothing. 

The NPG employs a constant-input neural network generator \citep{creswell2018generative} to learn asymmetric variances during noise-based robust training (see Figure \ref{fig:unif} right). The generator outputs tensors for anisotropic $\sigma$ and $\mu$, where $\mu_i$ and $\sigma_i$ represent the mean offset and variance multiplier per pixel, respectively. Anisotropic noise $\epsilon'=\epsilon^\top\sigma+\mu$ is will be generated from base isotropic noise $\epsilon$ and added to the input for randomized smoothing.

% Different from the pattern-fixed anisotropic noise, the mean offset of the noise is not pre-assigned to $0$ but determined by the learning along with 
% the variances/scale parameters of the anisotropic noise. 

\vspace{0.05in}

\noindent\textbf{Architecture}. 
%To generate the parameters for the anisotropic noise, 
We adopt the generator architecture (a multi-layer perception) in Generative Adversarial Network (GAN) \citep{goodfellow2020generative} to design a novel neural network generator. Different from GAN, this NPG does not depend on the input data but depends on the entire dataset. Therefore, we fixed the input as constants. Following \citep{goodfellow2020generative}, the NPG consists of $5$ linear layers, and the first $4$ of them are followed by activation layers. The output will be then transformed by a hyperbolic tangent function with an amplification factor $\gamma$, i.e., $\gamma tanh(\cdot)$. This amplified hyperbolic tangent layer limits the value of the variances since an infinite value in the noise parameters will crash the training.  It is worth noting that for other tasks in Natural Language Processing or Audio Processing, other NPG structures, e.g., transformer, NNs, or even heuristic algorithms can be also designed to fit the targeted tasks.

% \vspace{0.05in}
\vspace{+0.05in}
\noindent\textbf{Loss Function}. 
%The parameter generator will be trained to generate proper parameters for anisotropic noise during the robust training. 
% The NPG and the classifier can be trained simultaneously to achieve the best synergy. To train the NPG towards a desired convergence, we need to design a proper loss function. Note that certified robustness with anisotropic noise can be measured by the ALM or certified radius. Then, the NPG training should aim to maximize the product $\sqrt[d]{\prod \sigma}R$ or $\min\{\sigma\}R$. We can design two variants of the loss function according to the goals. Specifically, we can increase either $\sqrt[d]{\prod \sigma}$ or $\min\{\sigma\}$. The former is equivalent to $\log \sqrt[d]{\prod \sigma} = \frac{1}{d}\sum\sigma=\textit{mean}(\sigma)$. Since the certified radius $R$ is a function of $\underline{p_A}$, improving the prediction accuracy over the noise can improve it, which is also the goal of training the smoothed classifier. The loss function can be formally defined as:
The NPG and classifier can be trained together for optimal synergy. Designing an appropriate loss function is crucial for guiding NPG to desired convergence, targeting enhanced certified robustness via the ALM or certified radius. We aim to maximize either $\sqrt[d]{\prod \sigma}R$ or $\min\{\sigma\}R$, leading to two loss function variants that focus on increasing $\sqrt[d]{\prod \sigma}$ (equivalent to $mean\{\sigma\}$ in log scale) or $\min\{\sigma\}$, respectively. As the certified radius $R$ is influenced by prediction accuracy against noise, enhancing this accuracy also boosts $R$, aligning with the smoothed classifier's training objectives.

\begin{equation}\scriptsize
    \mathcal{L}(\theta_f,\theta_g)= -\underbrace{\textit{mean/min}\{\sigma(\theta_g)\}}_{\textit{Variance Loss}} + \underbrace{\sum_{k=1}^N y_k \log\hat{y}_k(x+\epsilon^\top\sigma(\theta_g)+\mu(\theta_g), \theta_f, \theta_g)}_{\textit{Smoothing Loss}}
\end{equation}

where the variance loss can be $\textit{mean} \{\sigma(\theta_g)\}$ or $\textit{min} \{\sigma(\theta_g)\}$ for improving ALM or certified radius, respectively. $\theta_f$ and $\theta_g$ denote the model parameters of the classifier and parameter generator, respectively, $k$ denotes the prediction class, $N$ represents the total number of classes, $y_k$ denotes the label of input $x$, and $\hat{y}_k$ is the prediction of $y_k$. The training of the NPG for $\mu$ is also guided by the smoothing loss to improve the prediction over the dataset.

% \vspace{0.05in}

% \noindent\textbf{Universality}. The universality of noise and 
% % the universality 
% that of our robustness are different: the former focuses on using one noise to universally protect an entire dataset, and the latter focuses on defending against a wide range of perturbations (e.g., different $\ell_p$-norms) with various anisotropic noise distribution.

% \vspace{0.05in}

% \noindent\textbf{Examples}. We show the $\mu$ map and $\sigma$ map in Figure \ref{fig:universal examples}. We \hl{found} that the 

% \begin{figure}
%     \centering
%     \includegraphics[width=\linewidth]{fig/universal example.pdf}
%     \caption{Examples of parameter generation for universal anisotropic noise}
%     \label{fig:universal examples}
% \end{figure}

\subsection{Certification-wise Anisotropic Noise}

While dataset-wise anisotropic noise fine-tunes parameters during training for improved robustness, it doesn't fully account for the heterogeneity among different input samples. The certification also varies by input, being valid only for the certified input and the corresponding radius. Thus, we propose a certification-wise NPG that generates tailored anisotropic noise for each sample, considering heterogeneity across both inputs and data dimensions.
Unlike dataset-wise noise, the parameter generators for $\mu$ and $\sigma$ in certification-wise NPG are cascaded, not parallel (see right in Figure \ref{fig:unif}). The mean parameter generator first processes the input $x$ to produce a $\mu$ map, followed by the variance generator using $x+\mu$ to generate a $\sigma$ map. Then the smoothed classifier is based on the generated $\mu$ and $\sigma$ through the classification.

% \begin{figure}[!h]
%     \centering
%     \includegraphics[width=0.6\linewidth]{fig/framework PersNoise.pdf}
%     \caption{Generating input-dependent anisotropic noise}%\vspace{-0.1in}
%     \label{fig:framework}
% \end{figure}
% \vspace{-5pt}

\begin{figure}[!th]
    \centering
    \includegraphics[width=0.95\linewidth]{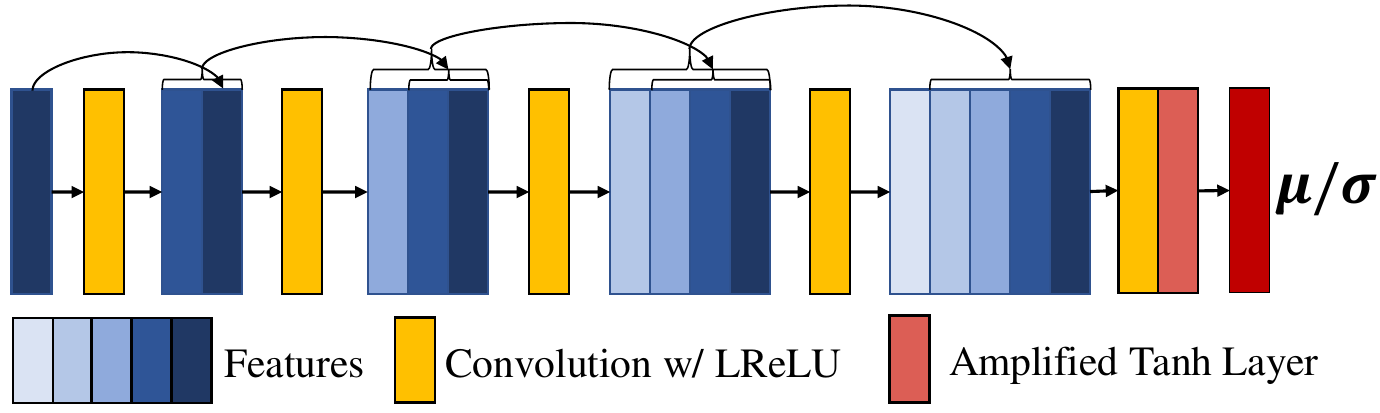}
    \caption{Architecture of parameter generator for certification-wise anisotropic noise.}\vspace{-0.1in}
    \label{fig:noisegenerator}
\end{figure}

\vspace{+0.05in}
\noindent\textbf{Architecture and Loss Function}. This NPG learns the mapping from the image to the $\mu$ and $\sigma$ maps, which is similar to the function of neural networks in image transformation tasks. Hence, inspired by the techniques used in image super-resolution \citep{zhang2018residual}, we also adopt the ``dense blocks'' \cite{huang2017densely} as the main architecture to design the NPG (see Figure \ref{fig:noisegenerator}). It consists of 4 convolutional layers followed by leaky-ReLU \cite{xu2015empirical}. Similar to the generator in dataset-wise anisotropic noise, the output is rectified by the amplified hyperbolic tangent function to stabilize the overall training process. Note that our parameter generator is a relatively small network (5 layers), thus it can be plugged in before any classifier for generating the certification-wise anisotropic noise without consuming too many computing resources (see Section \ref{apd: efficiency} for a detailed discussion on running time). For both $\mu$ and $\sigma$, we train a corresponding parameter generator for each using the same architecture. The loss function is similar to that used for the dataset-wise anisotropic noise, but the NPG takes $x$ as input and outputs $\mu$ and $\sigma$.

\subsection{Practical Algorithms}
\label{apd: algorithms}

Following Cohen et al. \cite{cohen2019certified}, we also use the Monte Carlo algorithm to bound the prediction probabilities of smoothed classifier and compute the ALM (certified region). Different from Cohen et al. \cite{cohen2019certified}, our noise distributions are either pre-assigned (as pattern-wise) or produced by the parameter generator (either dataset-wise or certification-wise). Our algorithms for certification and prediction using different noise generation methods are summarized in Algorithm \ref{alg:certify} and \ref{alg:predict} (w.l.o.g., taking the binary classifier as an example).

For simplicity of notations, the generation of anisotropic $\mu$ and $\sigma$ are summarized by the noise generation method(s) $M$. In the case of pattern-wise anisotropic noise, $M$ outputs pre-assigned fixed variance and zero-means; in case of dataset-wise and certification-wise anisotropic noise, $M$ adopts the parameter generators to generate the mean and variance maps. In the certification (Algorithm \ref{alg:certify}), we select the top-$1$ class $\hat{c_A}$ by the \textsc{ClassifySamples} function, in which the base classifier outputs the prediction on the noisy input sampled from the noise distribution. Once the top-$1$ class is determined, classification will be executed on more samples and the \textsc{LowerConfBound} function will output the lower bound of the probability $\underline{p_A}'$ computed by the Binomial test. If $\underline{p_A}'>\frac{1}{2}$, we output the prediction class and the ALM (measuring the certified region). Otherwise, it outputs ABSTAIN. In the prediction (Algorithm \ref{alg:predict}), we also generate the noise and then compute the prediction counts over the noisy inputs. If the Binomial test succeeds, then it outputs the prediction class. Otherwise, it returns ABSTAIN.

\begin{algorithm}[!t]\small
   \caption{UCAN-Certification}
   \label{alg:certify}
    \begin{algorithmic}[1]
   \Statex {\bfseries Given:} Base classifier $f$, anisotropic noise generation method $M$, input (e.g., image) $x$, number of Monte Carlo samples $n_0$ and $n$, confidence $1-\alpha$
    % \STATE {\bfseries Output:}
   \State $\mathbf{\mu}$, $\mathbf{\sigma} \leftarrow M$
   \State $counts\_select \leftarrow $\textsc{ClassifySamples}$(f,x,\mu,\mathbf{\sigma},n_0)$
   \State $\hat{c}_A \leftarrow \text{\textbf{top index in}} \  counts\_select$
   \State $counts \leftarrow$ \textsc{ClassifySamples}$(f,x,\mu,\mathbf{\sigma},n)$
   \State $\underline{p_A}' \leftarrow$ \textsc{LowerConfBound}$(counts[\hat{c}_A],n,1-\alpha)$
    \If{$\underline{p_A}'>\frac{1}{2}$}
        \State \textbf{return} prediction $\hat{c}_A$ and certified radius $\min \{\sigma\}$R/ALM $\sqrt[d]{\prod_{i}\sigma_i}R$
    \Else
        \State \textbf{return} ABSTAIN
    \EndIf
   
\end{algorithmic}
\end{algorithm}

\begin{algorithm}[!t]\small
\caption{UCAN-Prediction}
\label{alg:predict}
\begin{algorithmic}[1]
\Statex \textbf{Given}: Base classifier $f$, anisotropic noise generation method $M$, input (e.g., image) $x$, number of Monte Carlo samples $n$, confidence $1-\alpha$
% \State {\bfseries Output:}
\State $\mathbf{\mu}$, $\mathbf{\sigma} \leftarrow M$
\State $counts \leftarrow$ \textsc{ClassifySamples}$(f,x,\mu,\mathbf{\sigma},n)$
\State $\hat{c}_A \leftarrow \text{\textbf{top index in}} \  counts$
\State $n_A \leftarrow counts[\hat{c}_A]$
\If{\textsc{BinomialPValue}$(n_A,n,0.5) \le \alpha$}
    \State \textbf{return} prediction $\hat{c}_A$
\Else
    \State \textbf{return} ABSTAIN
\EndIf
   
\end{algorithmic}
\end{algorithm}

\section{Soundness Analysis for Certification-Wise Anisotropic Randomized Smoothing}
\label{sec:soundness}
In this section, we address the potential soundness pitfall in existing input-dependent randomized smoothing methods and demonstrate how our certification-wise approach provides enhanced robustness guarantees. Specifically, we provide formal definitions, theorems, and detailed proofs to establish the soundness of our method. Additionally, we carefully explain why randomized smoothing inherently depends on the clean input for reliable certification.

\subsection{Potential Soundness Pitfall in Existing Input-Dependent Randomized Smoothing Methods}

Recent attempts to enhance randomized smoothing have introduced input-dependent noise parameters that vary with both the input $x$ and potential adversarial perturbations $\delta$ \cite{eiras2021ancer, rumezhak2023rancer}. In these methods, the noise parameters $\mu(x, \delta)$ and $\sigma(x, \delta)$ are functions of both the clean input and the perturbation, leading to noise distributions that change based on the adversary's actions.

While such approaches aim to tailor the noise to each input and perturbation, it introduce potential concerns in the certified robustness guarantee. Specifically, when certifying a sample $x$, the noise is generated as $\epsilon(x)$, depending on $x$. However, when applying the guarantee to a potentially perturbed sample $x + \delta$, the noise changes to $\epsilon(x + \delta)$, which likely follows a different distribution from $\epsilon(x)$. Randomized smoothing requires that the prediction on the certified input $x$ and the perturbed input $x + \delta$ be based on the same smoothed classifier with the same noise distribution \cite{cohen2019certified}. Therefore, this input-dependent noise may affect the soundness of randomized smoothing. As noted in \cite{eiras2021ancer, rumezhak2023rancer}, an additional component (e.g., fixing the distribution or adopting memory-based certification) has been utilized to maintain valid guarantees by sacrificing the system performance, as the ``price paid for soundness''.

\subsection{Certification-wise Anisotropic Randomized Smoothing}

Different from \cite{eiras2021ancer, rumezhak2023rancer}, in our \emph{certification-wise} anisotropic randomized smoothing method, the noise parameters $\mu(x)$ and $\sigma(x)$ depend solely on the clean input $x$ and remain the same when applied to any potential perturbed samples during certification. This design ensures that the smoothed classifier remains consistent across all perturbations applied to $x$, inherently preserving the soundness of the certified robustness guarantee.

After generating the fine-tuned noise on the clean input $x$ (to be certified), our method constructs a robustness region that conceptually ensures robustness for $x$ against any perturbation $\delta$ within the certified radius, rather than actually injecting noise into the perturbed inputs. It is worth noting that the theorems in this paper universally work for isotropic and anisotropic noise independent of the noise generation method, which is sound for any randomized smoothing method. Furthermore, the ``pattern-wise'' and ``dataset-wise'' noise in Section~\ref{sec: Customize} have no potential concerns regarding this soundness problem.

\subsection{Certified Robustness Guarantee on Certification-wise Noise}

We now formally establish the soundness of our method by proving that the certified robustness guarantee holds under our certification-wise anisotropic noise.

\begin{theorem}[Certified Robustness of Certification-Wise Anisotropic Randomized Smoothing]
\label{thm:soundness}
Let $f: \mathbb{R}^d \to \mathcal{C}$ be any deterministic or randomized base classifier. Suppose that for the random variable $X = x + \epsilon$, where $\epsilon$ is drawn from an isotropic distribution, the certified radius function is $R(\cdot)$. Define the anisotropic random variable $Y = x + \epsilon^\top \sigma(x) + \mu(x)$, where $\sigma(x) \in \mathbb{R}^d$ and $\mu(x) \in \mathbb{R}^d$ are functions of $x$ only. If there exist $c_A \in \mathcal{C}$ and bounds $\underline{p_A}, \overline{p_B} \in [0,1]$ such that:
\begin{equation}
\label{eq:anisotropic_condition}
    \mathbb{P}_{\epsilon}\left( f(Y) = c_A \right) \geq \underline{p_A} \geq \overline{p_B} \geq \max_{c \neq c_A} \mathbb{P}_{\epsilon}\left( f(Y) = c \right)
\end{equation}
then, for any perturbation $\delta \in \mathbb{R}^d$ satisfying:
\begin{equation}
\label{eq:anisotropic_radius}
    \left\| \delta \oslash \sigma(x) \right\|_p \leq R\left( \underline{p_A}, \overline{p_B} \right)
\end{equation}
the smoothed classifier $g$ will consistently predict class $c_A$ at $x + \delta$, i.e., $g(x + \delta) = c_A$. Here, $\oslash$ denotes element-wise division, and $\| \cdot \|_p$ denotes the $\ell_p$ norm.
\end{theorem}

\subsection{Proof of Theorem~\ref{thm:soundness}}

\begin{proof}
Let $X = x + \epsilon$, where $\epsilon \in \mathbb{R}^d$ follows an isotropic noise distribution. Define the anisotropic random variable $Y = x + \epsilon^\top \sigma(x) + \mu(x)$.

Define a transformation $h_x: \mathbb{R}^d \to \mathbb{R}^d$ specific to input $x$:
\begin{equation}
\label{eq:transform_h}
    h_x(z) = z^\top \sigma(x) + \mu(x)
\end{equation}
where the multiplication and addition are element-wise.

Given any deterministic or randomized function $f : \mathbb{R}^d \to \mathcal{C}$, consider the composed classifier $f' = f \circ h_x$, mapping $z \mapsto f(h_x(z))$.

Under the transformation, the condition on the class probabilities becomes:
\begin{align}
    \mathbb{P}_{\epsilon}\left( f'(X) = c_A \right) &= \mathbb{P}_{\epsilon}\left( f(h_x(X)) = c_A \right) \\
    &= \mathbb{P}_{\epsilon}\left( f(Y) = c_A \right) \geq \underline{p_A} \\
    \max_{c \neq c_A} \mathbb{P}_{\epsilon}\left( f'(X) = c \right) &= \max_{c \neq c_A} \mathbb{P}_{\epsilon}\left( f(Y) = c \right) \leq \overline{p_B}
\end{align}
Thus, the probability bounds required for certification are satisfied by $f'$ under isotropic noise $\epsilon$.

From the standard randomized smoothing theory (e.g., Theorem~\ref{thm:cohen's thm}), since the transformed classifier $f'$ satisfies the probability bounds with respect to the isotropic noise $\epsilon$, we have that for any perturbation $\delta' \in \mathbb{R}^d$ satisfying $\| \delta' \|_p \leq R(\underline{p_A}, \overline{p_B})$, the prediction of $f'$ remains constant:
\begin{equation}
\label{eq:isotropic_guarantee}
    \arg \max_{c \in \mathcal{C}} \mathbb{P}_{\epsilon}\left( f'(x + \delta' + \epsilon) = c \right) = c_A
\end{equation}

Now, consider a perturbation $\delta \in \mathbb{R}^d$ in the original input space such that $\delta' = \delta \oslash \sigma(x)$. Then, the perturbed input after transformation is:
\begin{align}
    h_x(x + \delta) &= (x + \delta)^\top \sigma(x) + \mu(x) \\
    &= x^\top \sigma(x) + \delta^\top \sigma(x) + \mu(x) \\
    &= h_x(x) + \delta^\top \sigma(x)
\end{align}

Substituting back, we have:
\begin{align}
    g(x + \delta) &= \arg\max_{c \in \mathcal{C}} \mathbb{P}_{\epsilon}\left( f\left( x + \delta + \epsilon^\top \sigma(x) + \mu(x) \right) = c \right) \\
    &= \arg\max_{c \in \mathcal{C}} \mathbb{P}_{\epsilon}\left( f\left( h_x(x) + \delta^\top \sigma(x) + \epsilon^\top \sigma(x) \right) = c \right)
\end{align}

Since $\delta^\top \sigma(x) + \epsilon^\top \sigma(x) = (\delta + \epsilon)^\top \sigma(x)$, we have:
\begin{align}
    g(x + \delta) &= \arg\max_{c \in \mathcal{C}} \mathbb{P}_{\epsilon}\left( f\left( h_x(x) + (\delta + \epsilon)^\top \sigma(x) \right) = c \right) \\
    &= \arg\max_{c \in \mathcal{C}} \mathbb{P}_{\epsilon}\left( f'\left( x + \delta' + \epsilon \right) = c \right)
\end{align}

By the robustness of $f'$ under isotropic noise (Equation~\eqref{eq:isotropic_guarantee}), we have $g(x + \delta) = c_A$ whenever $\| \delta' \|_p = \| \delta \oslash \sigma(x) \|_p \leq R(\underline{p_A}, \overline{p_B})$.

Therefore, the smoothed classifier $g$ maintains its prediction $c_A$ within the certified region defined by the anisotropic noise parameters, establishing the soundness of our method.
\end{proof}

\subsection{Why Randomized Smoothing Depends on the Clean Input}

Randomized smoothing inherently depends on the clean input $x$ because the certification process aims to guarantee the classifier's robustness for that specific input. The certified radius $R(\underline{p_A}, \overline{p_B})$ is computed based on the class probabilities at $x$, which are estimated using noise added to $x$. Consequently, the smoothed classifier $g$ and the corresponding robustness guarantee are tied to the clean input.

In our method, the noise parameters $\sigma(x)$ and $\mu(x)$ are functions of $x$, further emphasizing this dependency. By designing the noise to be input-specific but independent of perturbations, we ensure that the certification process accurately reflects the classifier's behavior around the clean input, providing a meaningful and sound robustness guarantee.

\section{Experiments}
\label{sec: Experiments}
% We comprehensively evaluate UCAN in this section. Specifically, in Section \ref{sec: exp ani vs iso}, we test UCAN with the three different ways of generating anisotropic noise and compare them with the randomized smoothing baselines with isotropic noise. In Section \ref{sec:exp universality}, we thoroughly evaluate the universality of UCAN, including the universality on noise distributions and against different $\ell_p$ perturbations. In Section \ref{sec:best}, we benchmark the best performance of UCAN with the SOTA RS methods. More experiments and visualization can be found in Appendix \ref{apd:more exps}-\ref{apd: efficiency}

We present a comprehensive evaluation of UCAN in this section. Specifically, in Section \ref{sec: exp ani vs iso}, UCAN's performance with three anisotropic noise types is tested against isotropic noise baselines. Section \ref{sec:exp universality} assesses UCAN's universality regarding noise distributions and resistance to various $\ell_p$ perturbations. In Section \ref{sec:best}, we compare UCAN's top performance with state-of-the-art randomized smoothing methods. Section \ref{apd: visualization} and \ref{apd: efficiency} present the visualization and the efficiency. Additional experiments are detailed in Appendix \ref{apd:more exps}.

\vspace{+0.05in}
\noindent\textbf{Metrics}. 
% Existing randomized smoothing works usually follow \cite{cohen2019certified} to evaluate the certified robustness of randomized smoothing, in which the \emph{approximate certified test set accuracy} is adopted. 
% It is defined as the fraction of the test set that is certified correctly with a radius larger than $R$. In this paper, besides the certified accuracy per the certified radius, we also present the certified accuracy per the ALM, which can be defined as the fraction of the test set that is certified correctly \emph{with at least ALM}. % As a result, 
% Formally, in the context of anisotropic RS, certified accuracy w.r.t. certified radius and ALM can be defined as: $Acc( \min \{\sigma\} R)=\frac{1}{N}{\sum_{j=1}^N \mathbf{1}_{[g'(x^j+\delta)=y^j]}},  \, \forall \  ||\delta||_p \le \min \{\sigma\} R$ and $Acc(ALM)=\frac{1}{N}{\sum_{j=1}^N \mathbf{1}_{[g'(x^j+\delta)=y^j]}},  \, \forall \  ||\delta||_p \le \sqrt[d]{\prod \sigma_i}R$, respectively, where $x^j$ and $y^j$ denote the $j$-th sample and its label in the test set, respectively. $N$ denotes the number of inputs/images in the test set. In the context of isotropic RS, these two metrics merge as one: $Acc( \min \{\sigma\} R)=Acc(ALM)=Acc(R)$.
%\vspace{0.05in}
Existing randomized smoothing methods often adopt the \emph{approximate certified test set accuracy} from \cite{cohen2019certified}, defined as the proportion of the test set correctly certified above a radius $R$. Besides, we also report certified accuracy based on the ALM, representing the fraction of the test set certified correctly with \emph{at least ALM}. Formally, for asymmetric RS, we define $Acc( \min \{\sigma\} R)=\frac{1}{N}{\sum_{j=1}^N \mathbf{1}_{[g'(x^j+\delta)=y^j]}},  \, \forall \  ||\delta||_p \le \min \{\sigma\} R$ and $Acc(ALM)=\frac{1}{N}{\sum_{j=1}^N \mathbf{1}_{[g'(x^j+\delta)=y^j]}},  \, \forall \  ||\delta||_p \le \sqrt[d]{\prod \sigma_i}R$. In isotropic RS, these metrics converge to $Acc(R)$.

To fairly position our methods, when compared to the SOTA methods, we present the certified accuracy w.r.t. the certified radius and optionally w.r.t. ALM.

% \vspace{0.05in}
\vspace{0.05in} 
\noindent\textbf{Experimental Settings}. All the experiments are performed on three datasets: MNIST \citep{lecun2010mnist}, CIFAR10, \citep{krizhevsky2009learning} and ImageNet \citep{ILSVRC15}. Following \cite{cohen2019certified}, we obtain the certified accuracy on the entire test set in CIFAR10 and MNIST while randomly picking $500$ samples in the test set of ImageNet; we set $\alpha=0.001$ and the numbers of Monte Carlo samples $n_0=100$ and $n=100,000$. We use the original size of the images in MNIST and CIFAR10, i.e., $28\times 28$ and $3\times 32 \times 32$, respectively. For the ImageNet dataset, we resize the images to $3\times 224 \times 224$. In the training, we train the base classifier and the parameter generator (if needed) with all the training set in three datasets. For the MNIST dataset, we use a simple two-layer CNN as the base classifier. For the CIFAR10 and ImageNet datasets, we use the ResNet110 and ResNet50 \cite{he2016deep} as the base classifier, respectively. Dataset-wise NPG uses a 5-layer MLP, and certification-wise NPG uses a 4-layer CNN, both trained with Adam optimizer (learning rate $1\times 10^{-2}$, batch size 128, 200 epochs). Noise parameters are constrained to be positive. 
% The MNIST results are shown in Appendix \ref{apd: MNIST results} due to space limitations.

%%\vspace{0.05in}
\vspace{0.05in}
\noindent\textbf{Experimental Environment}. All the experiments were performed on the NSF Chameleon Cluster \cite{keahey2020lessons} with Intel(R) Xeon(R) Gold 6230 CPU @ 2.10GHz, 128G RAM, and Tesla V100 SXM2 32GB.

\subsection{Anisotropic vs Isotropic Noise in Randomized Smoothing}
\label{sec: exp ani vs iso}

\begin{figure}
\centering
\begin{subfigure}[b]{0.51\linewidth}
    \centering
    \includegraphics[width=\linewidth]{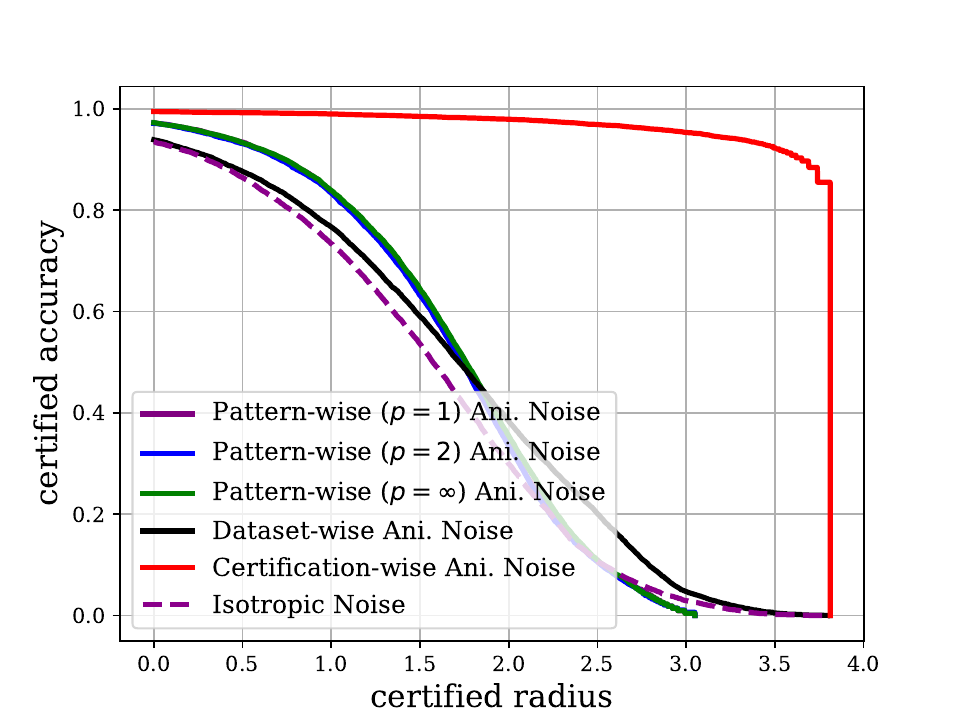}
    \caption{Acc vs. Radius (MNIST)}
    \label{fig:exp pre-assigned MNIST}
\end{subfigure}
\hspace{-0.2in} % Adjust or remove spacing based on your layout needs
\begin{subfigure}[b]{0.51\linewidth}
    \centering
    \includegraphics[width=\linewidth]{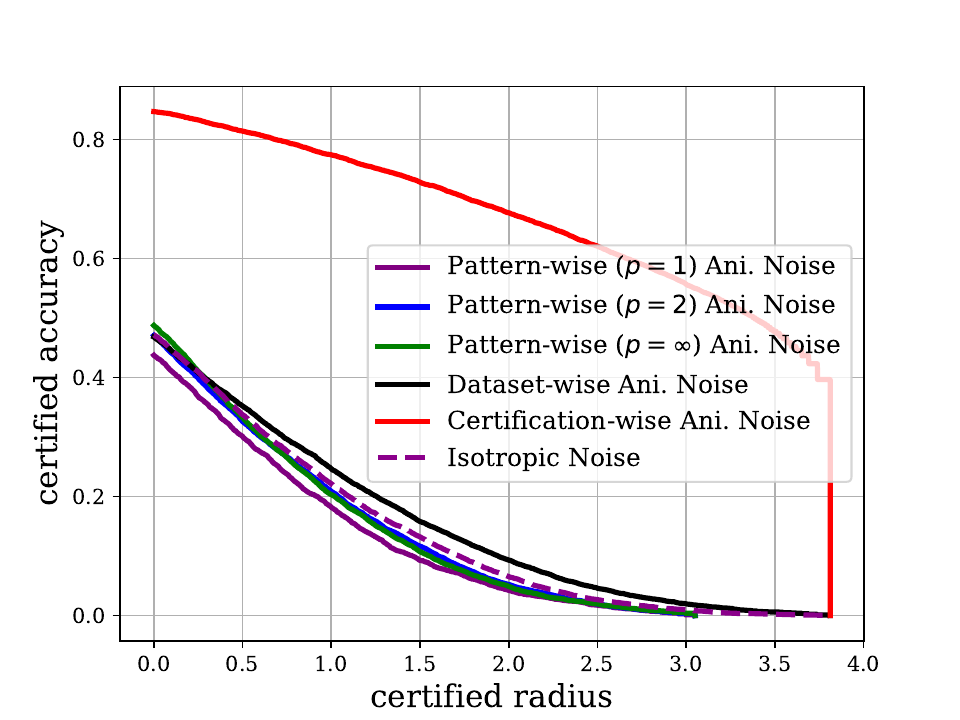}
    \caption{Acc vs. Radius (CIFAR10)}
    \label{fig:exp pre-assigned CIFAR10}
\end{subfigure}
% \hspace{-0.15in} % Adjust or remove spacing based on your layout needs
% \begin{subfigure}[b]{0.33\linewidth}
%     \centering
%     \includegraphics[width=\linewidth]{fig/Pre-assigned_imagenet.pdf}
%     \caption{Acc vs. Radius (ImageNet)}
%     \label{fig:exp pre-assigned ImageNet}
% \end{subfigure}

% \vspace{-1mm} % Adjust vertical spacing as needed

\begin{subfigure}[b]{0.51\linewidth}
    \centering
    \includegraphics[width=\linewidth]{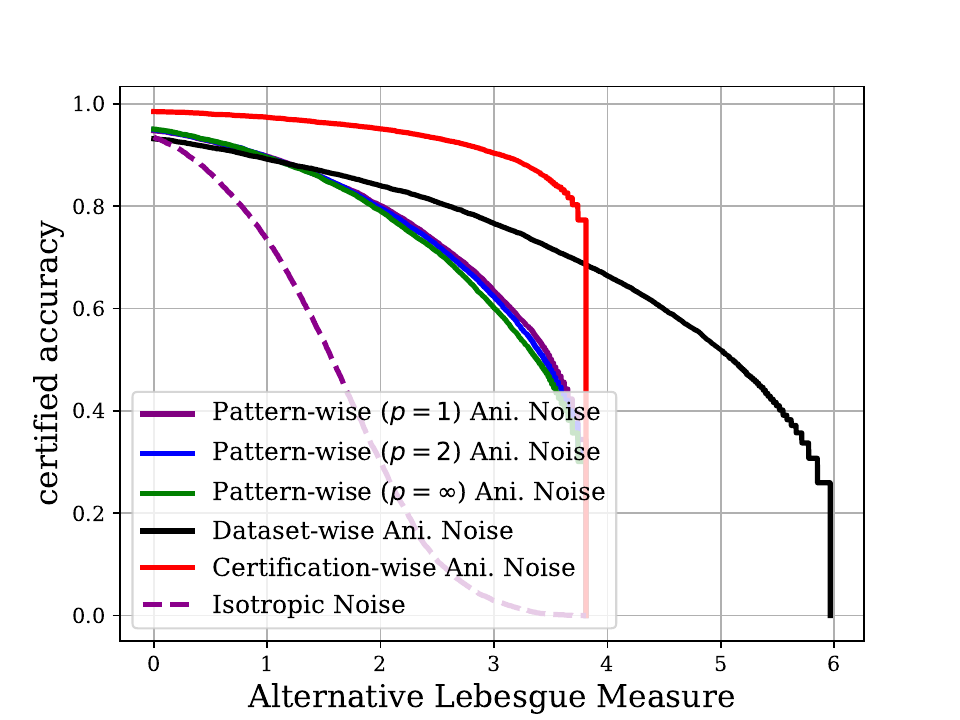}
    \caption{Acc vs. ALM (MNIST)}
    \label{fig:exp pre-assigned MNIST}
\end{subfigure}
\hspace{-0.2in} % Adjust or remove spacing based on your layout needs
\begin{subfigure}[b]{0.51\linewidth}
    \centering
    \includegraphics[width=\linewidth]{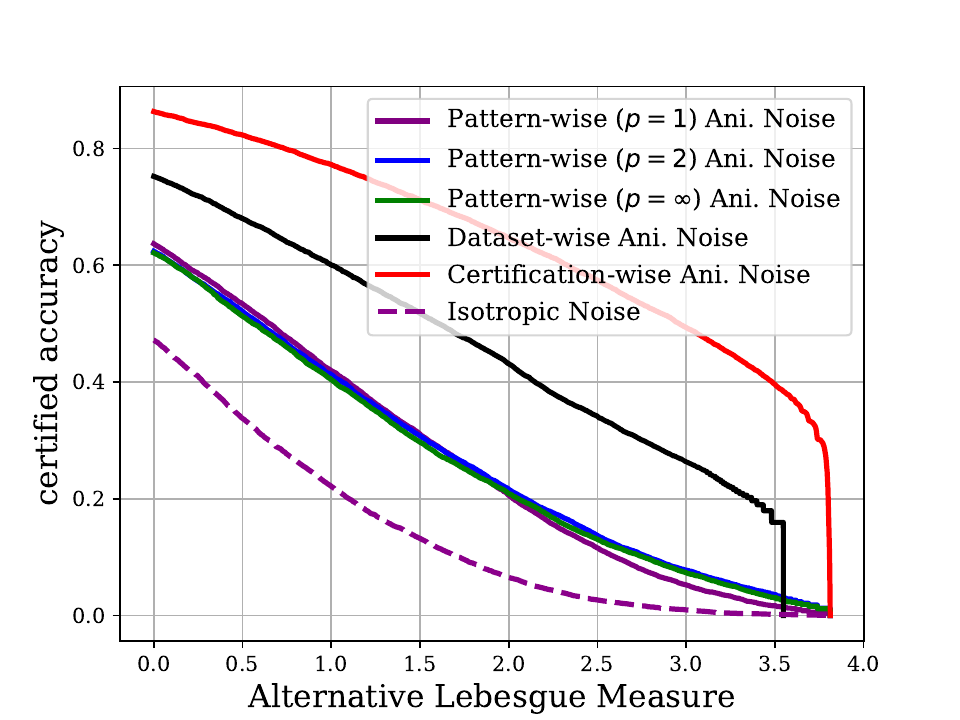}
    \caption{Acc vs. ALM (CIFAR10)}
    \label{fig:exp pre-assigned CIFAR10}
\end{subfigure}
\caption{RS with anisotropic noise vs. baseline \cite{cohen2019certified} with isotropic noise (Gaussian for $\ell_2$) -- UCAN gives significantly better certified accuracy and larger certified radius/ALM.}\vspace{-0.15in}
\label{fig:exp three methods}
\end{figure}

% \subsubsection{Pattern-fixed Anisotropic Noise}
% \label{sec:exp pre-assigned}

We first evaluate randomized smoothing with anisotropic noise generated by the three example NPGs. 
%For each dataset, we evaluate three types of spatial patterns for such pre-assigned noise. 
W.l.o.g., we adopt the most common setting as default: Gaussian distribution against $\ell_2$ perturbations, compare with the isotropic Gaussian baseline \cite{cohen2019certified}  (with zero-mean), which derives the tight certified radius (under multi-class setting) against $\ell_2$ perturbations. Other distributions against different $\ell_p$ perturbations (\emph{universality} of UCAN) are evaluated in Section \ref{sec:exp universality} and \ref{sec:best}. 
%can be examined similarly (detailed in Section \ref{sec:exp universality}). 
% We also evaluate randomized smoothing with universal anisotropic noise (derived from the entire dataset). Similarly, we compare the proposed UCAN with universal anisotropic noise with the isotropic baselines. Finally,  we evaluate the randomized smoothing with input-dependent anisotropic noise, which fully considers the heterogeneity over all the data dimensions. 

% Then, in this section, the anisotropic noise is generated based on the Gaussian distribution with different variance for different dimensions (zero-mean). Therefore, we compare the performance of randomized smoothing on anisotropic noise with the performance on randomized smoothing with isotropic Gaussian noise \cite{cohen2019certified}. 

\vspace{0.05in}
\noindent\textbf{Parameter Setting}. For a fair comparison, we follow \cite{cohen2019certified} to set the variance $\lambda=1$ for isotropic Gaussian noise to benchmark with our methods. For our pattern-wise method, since the variance varies in different dimensions, we re-scale $\sigma(a,b)$ such that $\textit{mean} \{\sigma\} \approx 1.0$. 
For the dataset-wise noise, we empirically select the $\gamma$ to achieve the best trade-off on each dataset.  
% After training, $\sqrt[d]{\prod \sigma_i}$ for MNIST, CIFAR10, and ImageNet are $1.56$, $0.93$, and $0.92$, respectively. 
For the certification-wise noise, the amplified factor $\gamma$ in the parameter generator is set as $1.0$ for all datasets.

\vspace{0.05in}
\noindent\textbf{Experimental Results}. The results on MNIST and CIFAR10 are presented in Figure \ref{fig:exp three methods}, and other results on ImageNet datasets are deferred to Appendix \ref{apd:more exps}. First, as shown in Figure \ref{fig:pre-assigned vis}, the certified accuracy of certification-wise anisotropic noise dominates all the noise customization methods across various settings. This indicates that optimizing anisotropic noise for each certification (of a specific input) can universally boost the performance w.r.t. the certified radius and the ALM since it achieves the best optimality on each certification (both the mean and variance can be optimized according to the input that will be certified). Second, dataset-wise anisotropic noise offers a modest improvement in certified accuracy w.r.t. the certified radius but significantly boosts certified accuracy w.r.t. the ALM. The reason is that the training of dataset-wise NPG can achieve a better trade-off between the prediction accuracy and the $\textit{mean} \{\sigma\}$ since NPG can learn to assign small variance to the key data dimensions to improve the prediction while maintaining the $\textit{mean} \{\sigma\}$ by increasing the variance in other data dimensions. However, it is hard to improve the trade-off between the prediction accuracy and the $\min\{\sigma\}$ since decreasing $\min\{\sigma\}$ drops the certified radius while improving the prediction (see some examples for the dataset-wise anisotropic noise in Appendix \ref{apd: visualization}). Similarly, the pattern-wise anisotropic noise only improves the certified accuracy w.r.t. certified radius slightly (even reduces the performance on CIFAR10), but we also observe a better trade-off between the certified accuracy w.r.t. ALM. Finally, we also observe that the dataset-wise anisotropic noise can significantly improve the ALM on MNIST (as much as $ALM=6$). These improvements stem from our method's ability to optimize noise parameters at different levels of optimality. Unlike fixed noise patterns, certification-wise anisotropic noise adapts both mean and variance to each input's characteristics, leading to better prediction accuracy while
  maintaining robustness guarantees. The significant ALM improvements (up to 6x on MNIST) result from our optimization targeting the overall certified volume rather than just the worst-case radius.

\subsection{Universality (Noise Distributions for $\ell_p$ Perturbations)}
\label{sec:exp universality}

% \YH{add some text in this page to fill the space up; making it nicer}

\begin{figure}
\centering
% \begin{subfigure}[b]{0.33\linewidth}
%     \centering
%     \includegraphics[width=\linewidth]{fig/L1_Preassigned_universality.pdf}
%     \caption{Pattern-fixed vs. $\ell_1$ pert.}
%     \label{fig:exp universality pre-assgiend l1}
% \end{subfigure}
% \hspace{-0.2in} % Adjust or remove spacing based on your layout needs
% \begin{subfigure}[b]{0.33\linewidth}
%     \centering
%     \includegraphics[width=\linewidth]{fig/L1_Universal_universality.pdf}
%     \caption{Universal vs. $\ell_1$ pert.}
%     \label{fig:exp universality universal l1}
% \end{subfigure}
% \hspace{-0.2in} % Adjust or remove spacing based on your layout needs
\begin{subfigure}[b]{0.51\linewidth}
    \centering
    \includegraphics[width=\linewidth]{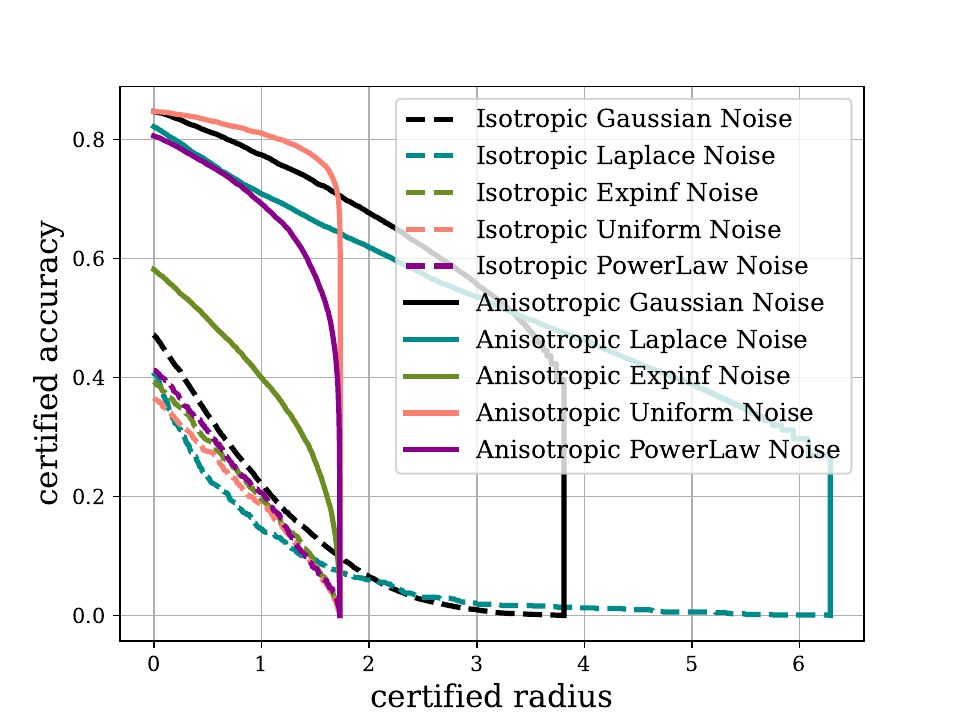}
    \caption{\scriptsize $\ell_1$ perturbation (Radius)}
    \label{fig:exp universality personalized l1}
\end{subfigure}
\hspace{-0.2in}
\begin{subfigure}[b]{0.51\linewidth}
    \centering
    \includegraphics[width=\linewidth]{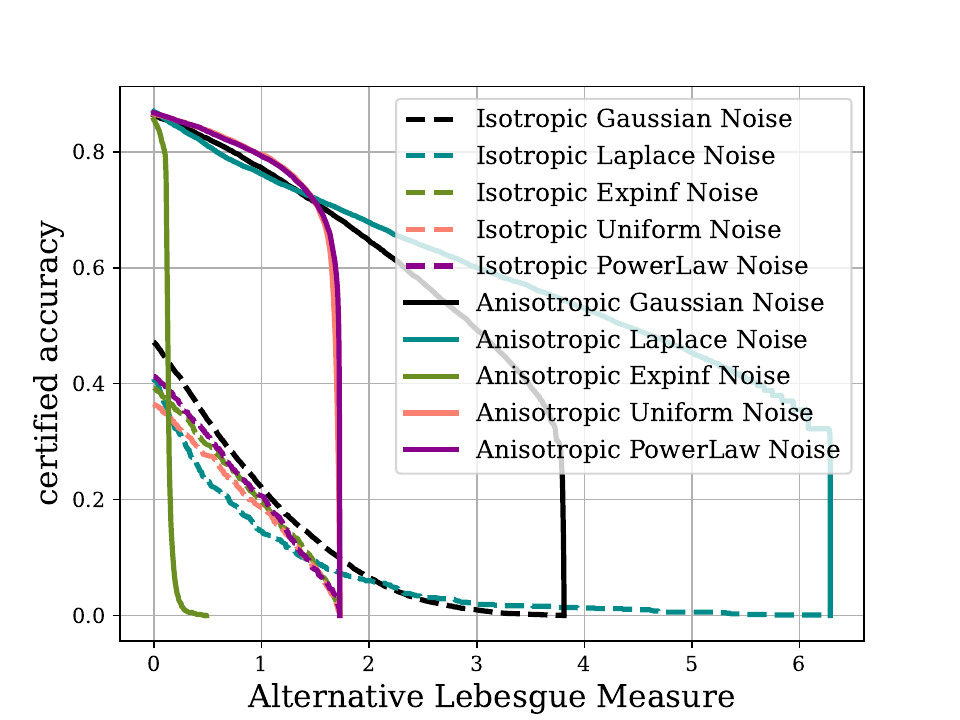}
    \caption{\scriptsize $\ell_1$ perturbation (ALM)}
    \label{fig:exp universality personalized l1}
\end{subfigure}

% \vspace{-4mm} % Adjust vertical spacing as needed
% \begin{subfigure}[b]{0.33\linewidth}
%     \centering
%     \includegraphics[width=\linewidth]{fig/Linf_Preassigned_universality.pdf}
%     \caption{Pattern-fixed vs. $\ell_\infty$ pert.}
%     \label{fig:exp universality pre-assigned linf}
% \end{subfigure}
% \hspace{-0.2in} % Adjust or remove spacing based on your layout needs
% \begin{subfigure}[b]{0.33\linewidth}
%     \centering
%     \includegraphics[width=\linewidth]{fig/Linf_Universal_universality.pdf}
%     \caption{Universal vs. $\ell_\infty$ pert.}
%     \label{fig:exp universality universal linf}
% \end{subfigure}
% \hspace{-0.2in} % Adjust or remove spacing based on your layout needs
\begin{subfigure}[b]{0.51\linewidth}
    \centering
    \includegraphics[width=\linewidth]{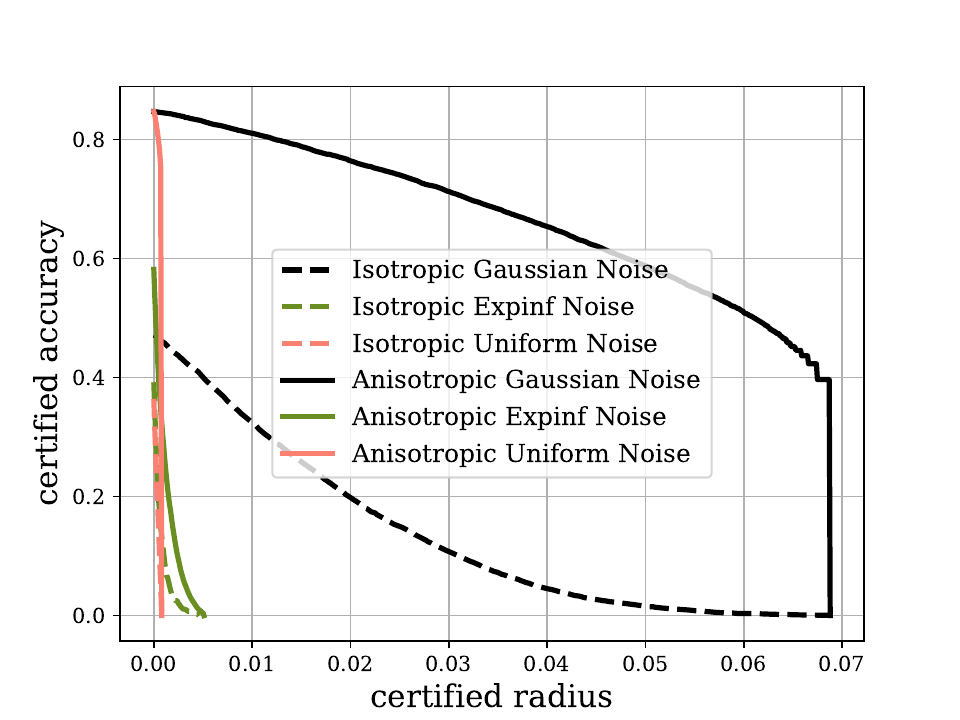}
    \caption{\scriptsize $\ell_\infty$ perturbation (Radius)}
    \label{fig:exp universality personalized linf}
\end{subfigure}
\hspace{-0.2in}
\begin{subfigure}[b]{0.51\linewidth}
    \centering
    \includegraphics[width=\linewidth]{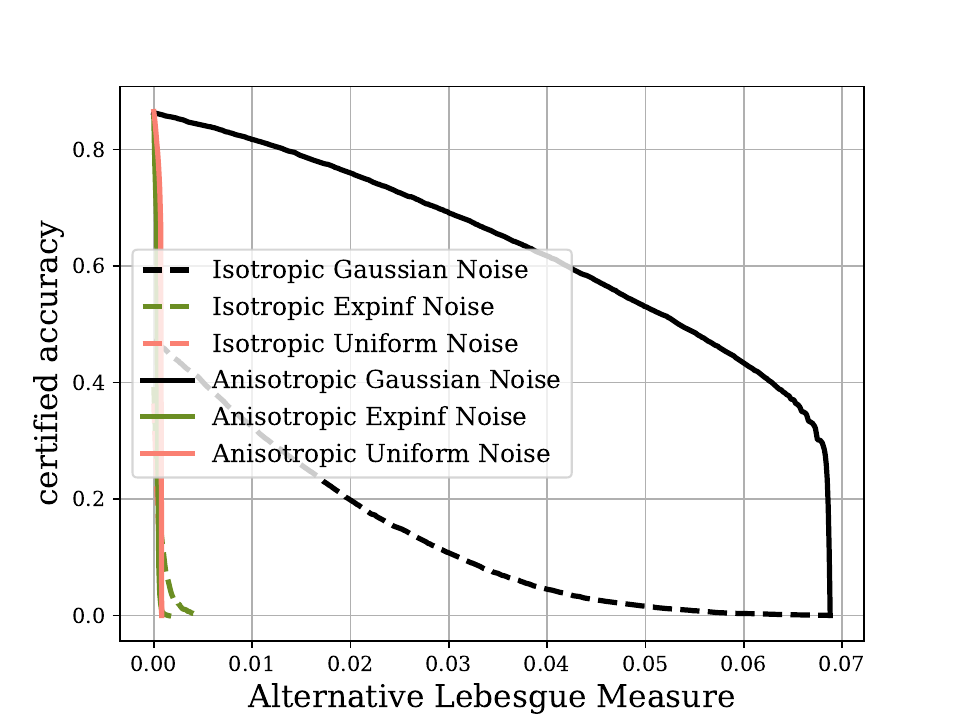}
    \caption{\scriptsize $\ell_\infty$ perturbation (ALM)}
    \label{fig:exp universality personalized linf}
\end{subfigure}

\caption{UCAN (certification-wise anisotropic) vs. isotropic randomized smoothing: different noise PDFs against various $\ell_p$ perturbations on CIFAR10.}\vspace{-0.15in}

\label{fig:exp universality}
\end{figure}

In this section, we evaluate the universality of UCAN with the certification-wise anisotropic noise over different noise distributions against different $\ell_p$ perturbations. 
% We have proposed the theory to transform any randomized smoothing with isotropic noise to randomized smoothing with anisotropic noise as well as a general and accurate metric to evaluate the certified robustness of randomized smoothing. We will illustrate how UCAN transforms randomized smoothing with different isotropic noise distribution to randomized smoothing with anisotropic noise, while boosting their performance. 
Specifically, we evaluate the randomized smoothing methods with noise listed in Table \ref{tab:transform}, %For a fair comparison,
%of different noise distributions against the same $\ell_p$-norm perturbations, 
and follow \cite{yang2020randomized} to set the scalar parameter $\lambda$ of different noise distributions with variance $1$. For the anisotropic noise, we follow the aforementioned parameter settings for pattern-wise, dataset-wise, and certification-wise anisotropic noise. We present the results against $\ell_1$ and $\ell_\infty$ perturbations due to the lack of existing RS theories for non-Gaussian noise against $\ell_2$ perturbations. %More RS methods based on the Gaussian noise against $\ell_2$ perturbations are presented in Section \ref{sec:best}.

In Figure \ref{fig:exp universality}, for all settings, UCAN can universally amplify the certified robustness of isotropic RS. It also shows that the anisotropic Laplace noise and the anisotropic Gaussian noise achieve the best trade-offs between certified accuracy and certified radius/ALM against $\ell_1$ perturbation and $\ell_\infty$ perturbation, respectively. 
% The input-dependent anisotropic noise achieves the best improvement over all the proposed methods due to the full consideration on the heterogeneity between the dimensions of the input for parameter generation.
% We also observe that the Laplace noise is the best anisotropic noise with the largest ALM (certified region) against $\ell_1$-norm perturbations. For $\ell_\infty$-norm perturbations, we observe that the Gaussian distribution significantly outperforms other noises. 
This universal improvement across different $\ell_p$ norms and noise distributions validates our linear transformation theory, which can seamlessly convert any isotropic randomized smoothing method to its anisotropic counterpart. The consistent performance gains demonstrate that our approach captures fundamental properties of optimal noise distribution regardless of $\ell_p$ norms and noise types.
More experiments with various NPG can be found in Appendix \ref{apd:more exps}.

\subsection{Best Performance Comparison vs. SOTA Methods}
\label{sec:best}

We also compare our best performance (certification with certification-wise anisotropic noise) with the best performance of 15 SOTA methods. Here we present the certified accuracy w.r.t. both \emph{ALM} and \emph{$\ell_p$ radius}. \emph{Note that the ALM and radius are equivalent for SOTA methods with isotropic noise.}
% \vspace{-0.2in}

 %\footnote{The variance loss used for training this method is modified to $\min{\sigma_i}$ since the isotropic radii depend on it as discussed in Corollary \ref{corollary}}.
% We only compare with randomized smoothing-based methods rather than all the certified robustness methods since randomized smoothing achieves the best scalability (it works on any arbitrary classifier and large-scale datasets). 
%Specifically, we compare UCAN with the original RS baseline \cite{cohen2019certified}, improved RS baselines \cite{salman2019provably,zhai2020macer,jeong2020consistency,jeong2021smoothmix,horvath2021boosting,yang2021certified}, generalized RS baselines \cite{yang2020randomized,zhang2020black}, data-dependent RS baselines \cite{alfarra2020data, sukenik2021intriguing}, denoised RS baselines \cite{carlini2022certified,zhang2023diffsmooth}, and the anisotropic RS baseline \cite{eiras2021ancer}.
% among the baseline methods, \cite{cohen2019certified} derives the first tight certified radius for Gaussian noise, which is shown to have the best performance of certified robustness over a wide range of distributions \cite{yang2020randomized}. Also, \cite{zhang2020black} proposes an optimization-based randomized smoothing method with a special-designed distribution that outperforms the Gaussian noise against $\ell_2$-norm perturbations. \cite{alfarra2020data}, \cite{sukenik2021intriguing} and \cite{wang2020pretrain} are data-dependent randomized smoothing methods (against $\ell_2$-norm perturbations) that optimize the variance $\sigma_0$ according to different inputs.
Following the same settings in such existing randomized smoothing methods \cite{cohen2019certified,alfarra2020data, sukenik2021intriguing, wang2020pretrain}, we focus on the Gaussian noise against $\ell_2$ perturbations to benchmark with them. %We compare our UCAN with these state-of-the-art methods 
Results are shown in Table \ref{tab: comparison MNIST}, \ref{tab: comparison CIFAR10}, and \ref{tab: comparison ImageNet}. We also present the improvement of our method over the best baseline in percentage.

On all three datasets, UCAN significantly boosts the certified accuracy. For instance, it achieves the improvement of $142.5\%$, $182.6\%$, and $121.1\%$ over the best baseline on MNIST, CIFAR10, and ImageNet, respectively. UCAN also achieves the best trade-off between certified accuracy and radius/ALM (\emph{two complementary metrics}): 1) UCAN presents both larger radius/ALM and higher certified accuracy in general, and 2) On large radius/ALM, UCAN can still achieve high certified accuracy. We also observe that our certified accuracy is smaller than the SOTA performances at some low radius/ALM on MNIST and ImageNet, this is because of the lower training performance of the classifier. 
% Finally, some visualized examples of anisotropic vs. isotropic noise and the results for efficiency (running time overhead) are given in Appendix \ref{apd: visualization} and \ref{apd: efficiency}.
The substantial improvements (up to 182.6\%) at larger radii demonstrate our method's strength in maintaining high certified accuracy where traditional methods degrade rapidly. This advantage stems from the combination of the linear transformation approach and the certification-wise noise, which preserves the statistical properties of the original
  smoothing while enabling dimension-specific noise adaptation. The larger the certified region required, the more pronounced our method's benefits become, as anisotropic noise can better accommodate the heterogeneity across input dimensions.

\vspace{0.05in}

\noindent \textbf{Beyond $\ell_2$ Norm}. We also compare our certification-wise method's performance against $\ell_1$ and $\ell_{\inf}$ perturbation with existing baselines \footnote{Due to the limited number of prior studies on the $\ell_1$ and $\ell_\infty$ norms, we compare our method with two available baselines for each case.} (see Table~\ref{tab: comparison l1linf}).
Across both norms, UCAN consistently improves certified accuracy (CA) over isotropic RS. For $\ell_1$, anisotropic Gaussian yields the highest CA at small–medium radii (e.g., 72\% at $R{=}1.5$ vs.\ 55\% in \cite{yang2020randomized}), while anisotropic Laplace dominates at larger radii (e.g., 61\% at $R{=}2.0$ vs.\ $\le\!25\%$ in baselines), whereas anisotropic Uniform peaks early then collapses (0\% beyond $R{=}2.0$). For $\ell_\infty$, anisotropic Gaussian is uniformly best (e.g., 73\% at $6/255$ vs.\ 31–38\% in \cite{cohen2019certified,zhang2020black}), maintaining sizable margins up to $16/255$. These trends confirm our linear-transformation theory’s universality and show that choosing the noise family to match the threat norm (Laplace for $\ell_1$, Gaussian for $\ell_\infty$) and employing certification-wise NPGs yields the highest practical ``optimality'' under fixed budgets.

\begin{table}[!h]
    \centering
    \caption{Certified accuracy vs. $\ell_1$ and $\ell_{\inf}$ perturbations (CIFAR10). \colorbox{green!60}{Best} and \colorbox{red!40}{$\geq$ SOTA}}
    % \scriptsize
    % \vspace{-5pt}
     \resizebox{\linewidth}{!}{
    \begin{tabular}{ l  c c c c c c c c c c}
    \hline
    \hline
    Radius ($\ell_1$ norm)  &  0.50  &1.00 &1.50  & 2.00 & 2.5 & 3.0 & 3.5 & 4.0\\
    \hline 
       Teng et al.'s \cite{teng2020ell} & 61\% & 39\% & 24\% & 16\% & 11\% & 7\% & 4\% & 3\%  \\
      Yang et al.'s \cite{yang2020randomized}  & 74\% & 62\% & 55\% & 48\% & 43\% & 40\% & 37\% & 33\% \\
    \hline
    Ours (Ani. Uniform) & \cellcolor{green!60}{\textbf{84\%}}& \cellcolor{green!60}{\textbf{82\%}} &\cellcolor{green!60}{\textbf{76\%}} & 0\% & 0\% & 0\% & 0\% & 0\% \\
    Ours (Ani. Gaussian) & \cellcolor{red!40}{\textbf{81\%}}& \cellcolor{red!40}{\textbf{77\%}} &\cellcolor{red!40}{\textbf{72\%}} &\cellcolor{green!60}{\textbf{66\%}} &\cellcolor{green!60}{\textbf{61\%}} &\cellcolor{green!60}{\textbf{56\%}} &\cellcolor{red!40}{\textbf{47\%}} &0\% \\
    Ours (Ani. Laplace) & \cellcolor{red!40}{\textbf{75\%}}& \cellcolor{red!40}{\textbf{70\%}} &\cellcolor{red!40}{\textbf{66\%}} &\cellcolor{red!40}{\textbf{61\%}} &\cellcolor{red!40}{\textbf{57\%}} &\cellcolor{red!40}{\textbf{52\%}} &\cellcolor{green!60}{\textbf{50\%}} &\cellcolor{green!60}{\textbf{46\%}}  \\
    \hline
    \hline
        Radius ($\ell_{\inf}$ norm)  &  2/255  & 4/255 & 6/255  & 8/255 & 10/255 & 12/255 & 14/255 & 16/255\\
    \hline 
       Cohen et al.'s \cite{cohen2019certified} & 58\% & 42\% & 31\% & 25\% & 18\% & 13\% & -- & -- \\
      Zhang et al.'s \cite{zhang2020black}  & 60\% & 47\% & 38\% & 32\% & 23\% & 17\% & -- & -- \\
    \hline
    Ours (Ani. Gaussian) & \cellcolor{green!60}{\textbf{81\%}}& \cellcolor{green!60}{\textbf{78\%}} &\cellcolor{green!60}{\textbf{73\%}} &\cellcolor{green!70}{\textbf{70\%}} &\cellcolor{green!70}{\textbf{65\%}} &\cellcolor{green!60}{\textbf{60\%}} &\cellcolor{green!60}{\textbf{54\%}} &\cellcolor{green!60}{\textbf{48\%}}  \\
    \hline
    \hline
    \end{tabular}}
  %  }
        \label{tab: comparison l1linf}
         % \vspace{-3pt}
\end{table}

\begin{table*}[!th]
    \centering
    \caption{Certified accuracy vs. $\ell_2$ perturbations (MNIST). \colorbox{green!60}{Best} and \colorbox{red!40}{$\geq$ SOTA}}
    % \scriptsize
    % \vspace{-5pt}
    {\small
     \resizebox{0.95\linewidth}{!}{\begin{tabular}{ l  c c c c c c c c c c}
    \hline
    \hline
    Radius and ALM (equivalent for isotropic) & 0.0 & 0.25 & 0.50 & 0.75 &1.00 & 1.25 &1.50 & 1.75 & 2.00 & 2.25 \\
    \hline
       Cohen's \cite{cohen2019certified} & 83\% & 61\% & 43\% & 32\% & 22\% & 17\% & 14\% & 9\% & 7\% & 4\%  \\
    % \cite{zhang2020black}  & -- & 61\% & 46\% & 37\% & 25\% & 19\% & 16\% & 14\% &11\% & 9\%  \\
    %   \cite{alfarra2020data}  & 82\% & 68\% & 53\% & 44\% & 32\% & 21\% & 14\% & 8\% & 4\% & 1\% \\
      Sample-wise \cite{wang2020pretrain}  & 98\% & 97\% & 96\% & 93\% & 88\% & 81\% & 73\% & 57\% & 41\% & 25\% \\
      Input-depend \cite{sukenik2021intriguing}  & 99\% & 98\% & 97\% & 94\% & 88\% & 79\% & 58\% & 27\% & 0\% & 0\% \\
        MACER \cite{zhai2020macer} & 99\% & 99\% &96\% &95\% &90\% &83\% &73\% &50\% &36\% &28\% \\
        SmoothMix \cite{jeong2021smoothmix} & 99\% & 99\% &98\% &97\% &93\% &89\% &82\% &71\% &45\% &37\% \\
        DRT \cite{yang2021certified} & 99\% &98\% & 98\% &97\% &93\% &89\% &83\% &70\% &48\% &40\% \\
        \hline
        Ours (certified accuracy w.r.t. ALM) &98\% &98\% &\cellcolor{red!40}98\%   &\cellcolor{red!40}98\% &\cellcolor{red!40}97\% &\cellcolor{red!40}97\% &\cellcolor{red!40}96\%  &\cellcolor{red!40}96\% &\cellcolor{red!40}95\% &\cellcolor{red!40}94\%\\
        Improvement over the Best Baseline ($\%$) & -1.0\% & -1.0\% & +0.0\% &+1.0\% &+4.3\% &+9.0\% &+15.7\% &+35.2\% &+98.0\% &+135.0\% \\
        \hline
        Ours (certified accuracy w.r.t. radius) & \cellcolor{green!60}{\textbf{99\%}}& \cellcolor{green!60}{\textbf{99\%}} &\cellcolor{green!60}{\textbf{99\%}} &\cellcolor{green!60}{\textbf{99\%}} &\cellcolor{green!60}{\textbf{99\%}} &\cellcolor{green!60}{\textbf{99\%}} &\cellcolor{green!60}{\textbf{98\%}} &\cellcolor{green!60}{\textbf{98\%}} &\cellcolor{green!60}{\textbf{98\%}} &\cellcolor{green!60}{\textbf{97\%}} \\
        Improvement over the Best Baseline ($\%$) & +0.0\% & +0.0\% & +1.0\% &+2.1\% &+6.5\% &+11.2\% &+18.1\% &+38.0\% &+104.2\% &+142.5\% \\
    \hline
    \hline
    \end{tabular}}
    }
        \label{tab: comparison MNIST}
         % \vspace{-3pt}
\end{table*}

\begin{table*}[!th]
    \centering
    \caption{Certified accuracy vs. $\ell_2$ perturbations (CIFAR10). \colorbox{green!60}{Best} and \colorbox{red!40}{$\geq$ SOTA}}
    % \vspace{-0.1in}
    % \scriptsize
    {\small
     \resizebox{0.95\linewidth}{!}{\begin{tabular}{ l  c c c c c c c c c c}
    \hline
    \hline
    Radius and ALM (equivalent for isotropic) & 0.0 & 0.25 & 0.50 & 0.75 &1.00 & 1.25 &1.50 & 1.75 & 2.00 & 2.25 \\
    \hline
        Cohen's \cite{cohen2019certified} & 83\% & 61\% & 43\% & 32\% & 22\% & 17\% & 14\% & 9\% & 7\% & 4\%  \\
        SmoothAdv\cite{salman2019provably} & -- &81\% & 63\% & 52\% & 37\% & 33\% & 29\% & 25\% & 18\% & 16\% \\
        MACER \cite{zhai2020macer} & 81\% & 71\% & 59\% & 47\% & 39\% & 33\% & 29\% & 23\% & 19\% & 17\% \\
        Consistency \cite{jeong2020consistency} &78 \% &69\% & 58\% & 49\% & 38\% & 34\% & 30\% & 25\% & 20\% & 17\% \\
        SmooothMix \cite{jeong2021smoothmix} &77\% & 68\% & 58\% & 48\% & 37\% & 32\% & 26\% & 20\% & 17\% & 15\% \\
        Boosting \cite{horvath2021boosting}& 83\% & 71\% & 60\% & 52\% & 39\% & 34\% & 30\% & 25\% & 20\% & 17\% \\
        DRT \cite{yang2021certified} & 73 \% & 67\% &60\% &51\% &40\% &36\% &30\% &24\% &20\% & --\\
        Black-box \cite{zhang2020black}  & -- & 61\% & 46\% & 37\% & 25\% & 19\% & 16\% & 14\% &11\% & 9\%  \\
        Data-depend \cite{alfarra2020data}  & 82\% & 68\% & 53\% & 44\% & 32\% & 21\% & 14\% & 8\% & 4\% & 1\% \\
      Sample-wise \cite{wang2020pretrain}  & 84\% & 74\% & 61\% & 52\% & 45\% & 41\% &36\% & 32\% & 27\% & 23\%\\
      Input-depend \cite{sukenik2021intriguing}  & 83\% & 62\% & 43\% & 27\% & 18\% & 11\% & 5\% & 2\% & 0\% & 0\% \\
      Denoise 1 \cite{carlini2022certified} &80\% &70\% &55\% &48\% &37\% &32\% &29\% &25\% &15\% &14\% \\
      Denoise 2 \cite{zhang2023diffsmooth} &85\% &76\% &66\% &57\% &44\% &37\% &31\% &25\% &22\% &20\%\\
      ANCER \cite{eiras2021ancer} & 84\% &80\% & 67\%  & -- &34\% & -- & 15\% & -- &11\% & --\\
    RANCER \cite{rumezhak2023rancer}  & --& 81\% & 48\% &  28\% & 11\% & 1\% & -- & -- & -- &--  \\
        \hline
        Ours (certified accuracy w.r.t. ALM) &\cellcolor{green!60}{\textbf{86\%}} &\cellcolor{green!60}{\textbf{84\%}} & \cellcolor{green!60}{\textbf{82\%}} & \cellcolor{green!60}{\textbf{80\%}} & \cellcolor{red!40}74\% & \cellcolor{red!40}71\% & \cellcolor{red!40}68\% & \cellcolor{red!40}65\% & \cellcolor{red!40}61\% & \cellcolor{red!40}57\% \\
        Improvement over the Best Baseline ($\%$)& +1.2\% & +3.7\% &+6.5\% & +40.4\% &+39.6\% & +73.2\% &+88.9\% &+103.1\% &+125.9\% &+147.8\% \\
        \hline
        Ours (certified accuracy w.r.t. radius) &\cellcolor{red!40}85\% &\cellcolor{red!40}{83\%} &\cellcolor{red!40}81\% &\cellcolor{red!40}80\% &\cellcolor{green!60}{\textbf{77\%}} &\cellcolor{green!60}{\textbf{75\%}} &\cellcolor{green!60}{\textbf{73\%}} &\cellcolor{green!60}{\textbf{70\%}} &\cellcolor{green!60}{\textbf{68\%}} &\cellcolor{green!60}{\textbf{65\%}} \\
        Improvement over the Best Baseline ($\%$)& +0.0\% & +2.5\% &+5.2\% & +40.4\% &+45.3\% & +82.9\% &+102.8\% &+118.8\% &+151.9\% &+182.6\% \\
    \hline
    \hline
    \end{tabular}}
    }
      \label{tab: comparison CIFAR10}
     % \vspace{-3pt}
\end{table*}

\begin{table*}[!th]\scriptsize
    \centering
    \caption{Certified accuracy vs. $\ell_2$ perturbations (ImageNet). \colorbox{green!60}{Best} and \colorbox{red!40}{$\geq$ SOTA}}
    % \vspace{-0.1in}
    {\scriptsize
    \resizebox{\linewidth}{!}{
    \begin{tabular}{ l ccc c c c c c }
    \hline
    \hline
    Radius and ALM (equivalent for isotropic)    & 0.00 & 0.50 & 1.00 & 1.50 & 2.00 & 2.50 & 3.00 & 3.50\\
    \hline
        Cohen's \cite{cohen2019certified} & 67\% & 49\% & 37\% & 28\% & 19\% &15\% & 12\% & 9\% \\
        SmoothAdv \cite{salman2019provably} & 67\% & 56\%  & 45\% & 38\% & 28\% & 26\% & 20\% & 17\% \\
        MACER \cite{zhai2020macer} & 68\% & 57\% & 43\% & 37\% & 27\% & 25\% & 20\% & --\\
        Consistency \cite{jeong2020consistency} & 57\% & 50\% & 44\% & 34\% & 24\% & 21\% & 17\% & --\\
        SmoothMix \cite{jeong2021smoothmix} & 55\% & 50\% & 43\% & 38\% & 26\% & 24\% & 20\% & --\\
        Boosting \cite{horvath2021boosting} & 68\% & 57\% & 45\% & 38\% & 29\% & 25\% & 21\% & 19\% \\
        DRT\cite{yang2021certified} & 50\% & 47\% &44\% &39\% &30\% &29\% &23\% & --\\
        Black-box \cite{zhang2020black} & -- & 50\% & 39\% & 31\% & 21\% & 17\% & 13\% & 10\% \\
        Data-depend \cite{alfarra2020data} & 62\% & 59\% & 48\% &43\% &31\% &25\% &22\% &19\%  \\
        Denoise 1  \cite{carlini2022certified}  &48\% &41\% &30\% &24\% &19\% &16\% &13\% &-- \\
        Denoise 2 \cite{zhang2023diffsmooth} &66\% &59\% &48\% &40\% &31\% &25\% &22\% &--\\
        % ANCER \cite{eiras2021ancer} & 70\% & 70\%& 62\% & 61\% &42\% &36\% &29\% &-- \\
        ANCER \cite{eiras2021ancer} & 66\% & 66\%& 62\% & 58\% & 44\% &37\% &32\% & --\\
        \hline
        Ours (certified accuracy w.r.t. ALM)   & \cellcolor{green!60}{\textbf{71\%}} & \cellcolor{green!60}{\textbf{66\%}} & \cellcolor{green!60}{\textbf{62\%}} & \cellcolor{green!60}{\textbf{58\%}} & \cellcolor{green!60}{\textbf{54\%}} & \cellcolor{green!60}{\textbf{51\%}} & \cellcolor{green!60}{\textbf{47\%}} &\cellcolor{green!60}{\textbf{42\%}}\\
        Improvement over the Best Baseline ($\%$) & +4.4\% & +0.0\% &+0.0\% &+0.0\% &+22.7\% &+37.8\% & +46.9\% &+121.1\%\\
        \hline
        Ours (certified accuracy w.r.t. radius) &  65\% &62\% &58\% &53\% &\cellcolor{red!40}50\% &\cellcolor{red!40}46\% &\cellcolor{red!40}43\% &\cellcolor{red!40}38\%\\
        Improvement over the Best Baseline ($\%$) & -4.4\% & -6.1\% &-6.5\% &-8.6\% &+13.6\% &+24.3\% & +34.4\% &+100.0\%\\
        \hline
        \hline
        \end{tabular}}
    }
    \label{tab: comparison ImageNet}
\end{table*}

\subsection{Visualization}
\label{apd: visualization}

\begin{figure}[!t]
\centering
% Row 1
\begin{subfigure}[b]{0.45\linewidth}
    \centering
    \includegraphics[width=0.99\linewidth]{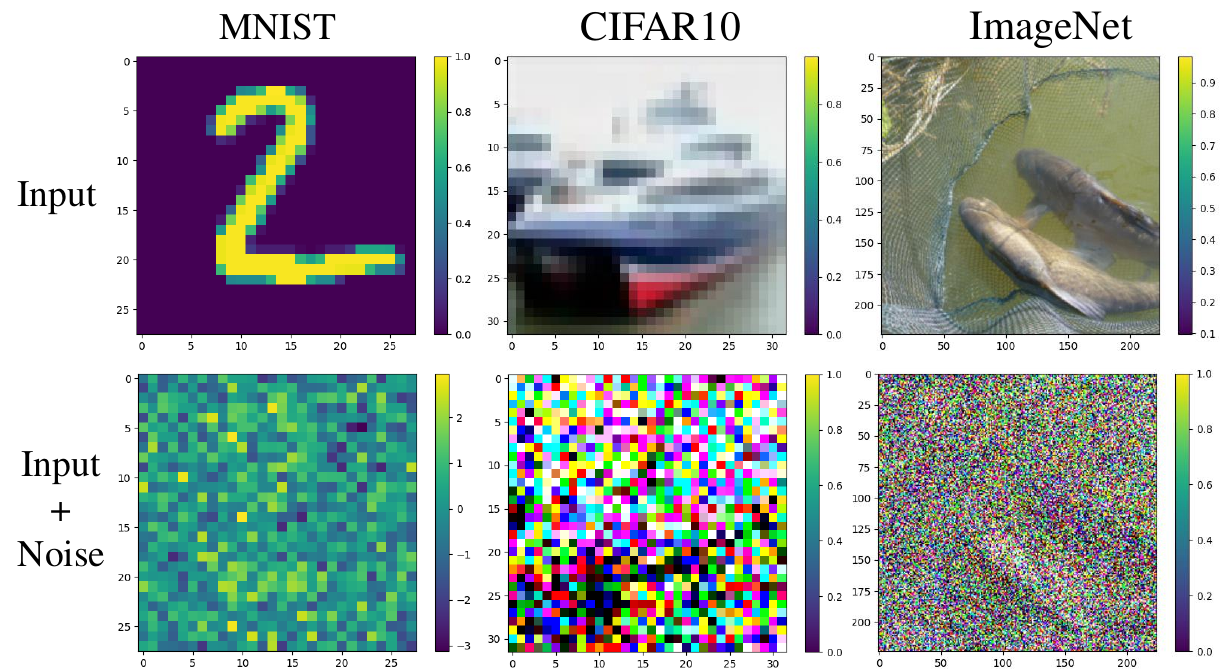}
    \caption{Example of Isotropic Noise}
    \label{fig: visual iso}
\end{subfigure}
\hspace{0.04\linewidth} % Adjust the space between the figures if necessary
\begin{subfigure}[b]{0.45\linewidth}
    \centering
    \includegraphics[width=0.99\linewidth]{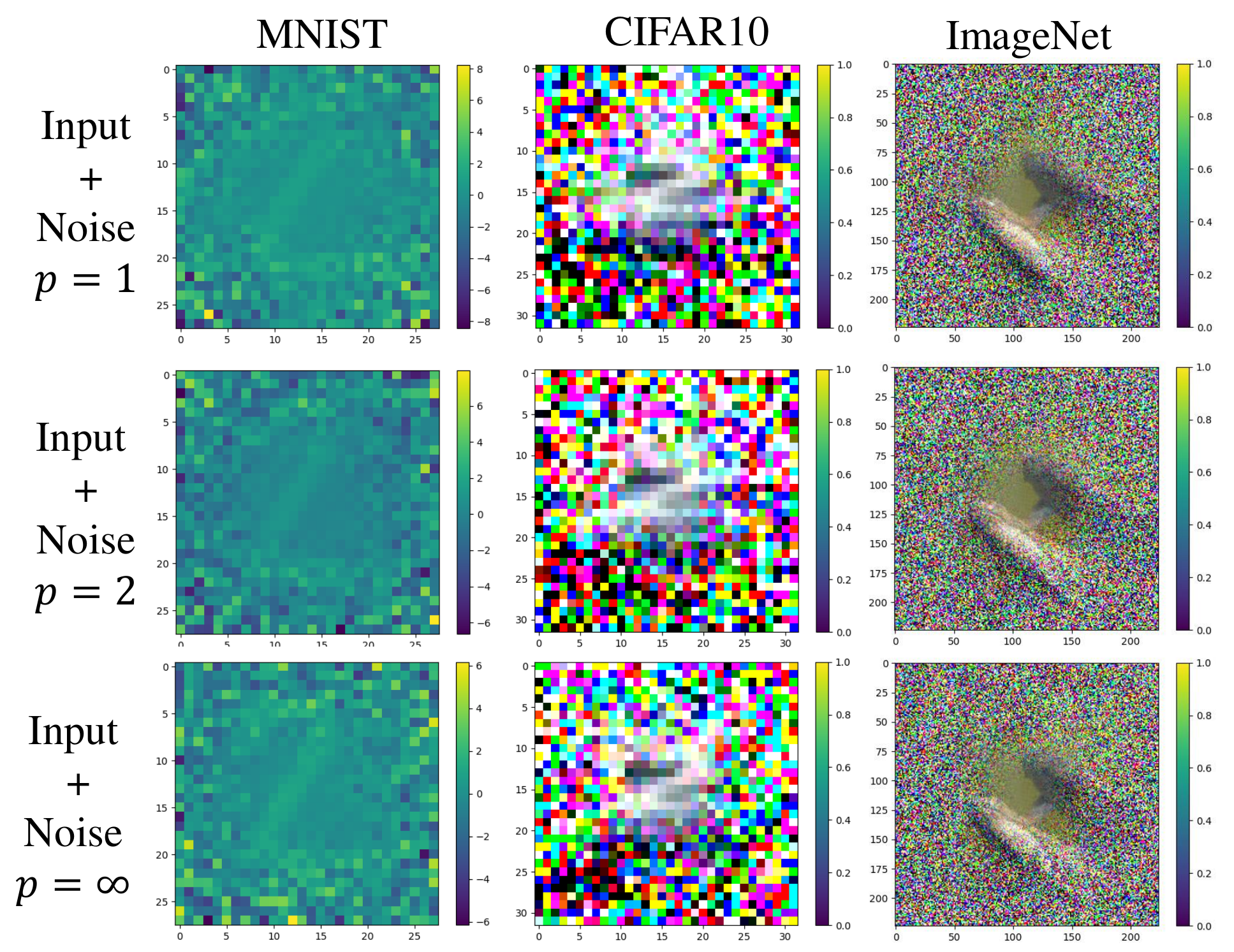}
    \caption{Examples of pattern-wise anisotropic noise}
    \label{fig: visual preassigned}
\end{subfigure}

\vspace{0.05in} % Adjust the vertical space between the rows if necessary

% Row 2
\begin{subfigure}[b]{0.45\linewidth}
    \centering
    \includegraphics[width=0.99\linewidth]{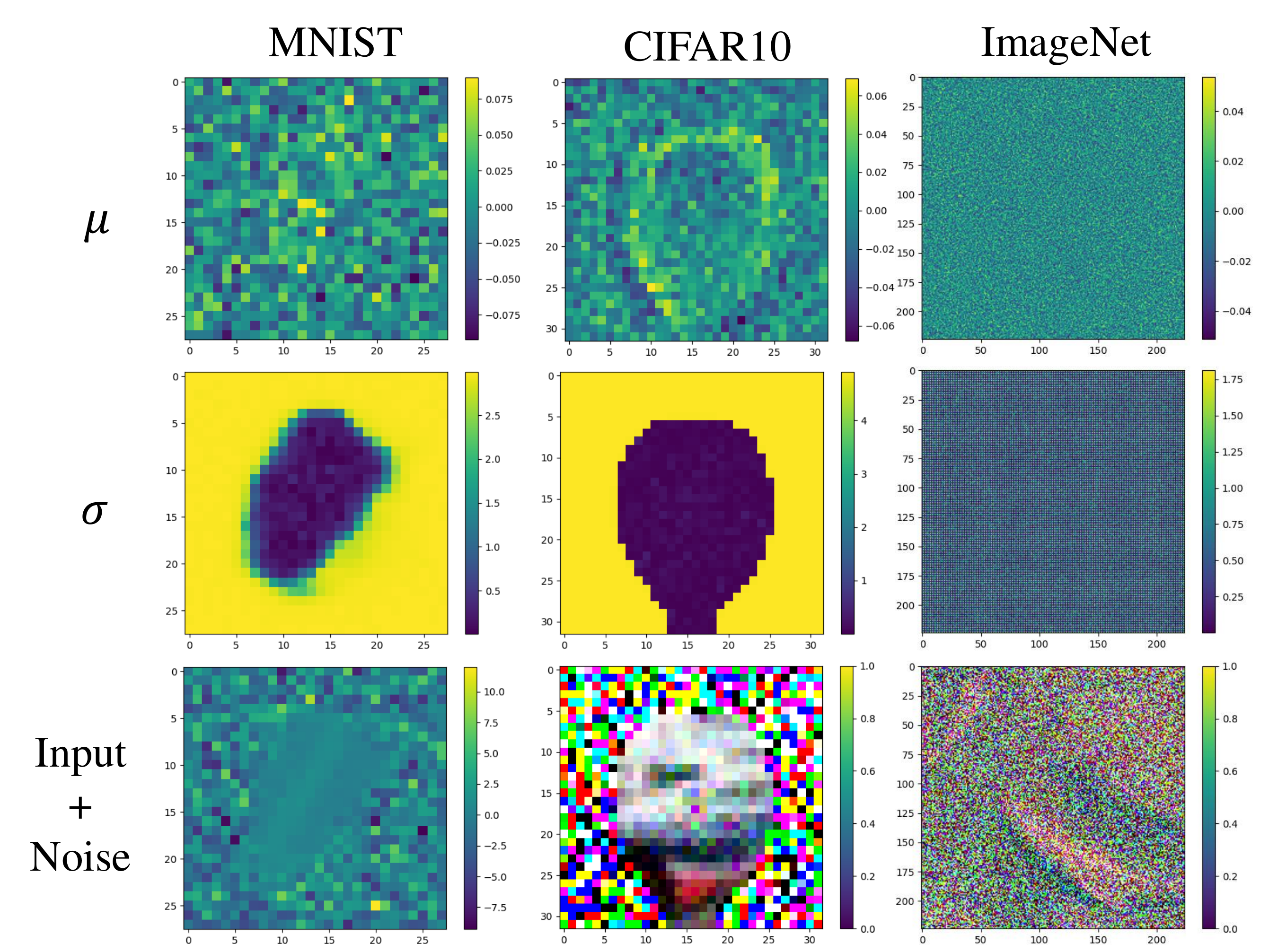}
    \caption{Examples of dataset-wise anisotropic noise}
    \label{fig: visual universal}
\end{subfigure}
\hspace{0.04\linewidth} % Adjust the space between the figures if necessary
\begin{subfigure}[b]{0.45\linewidth}
    \centering
    \includegraphics[width=0.99\linewidth]{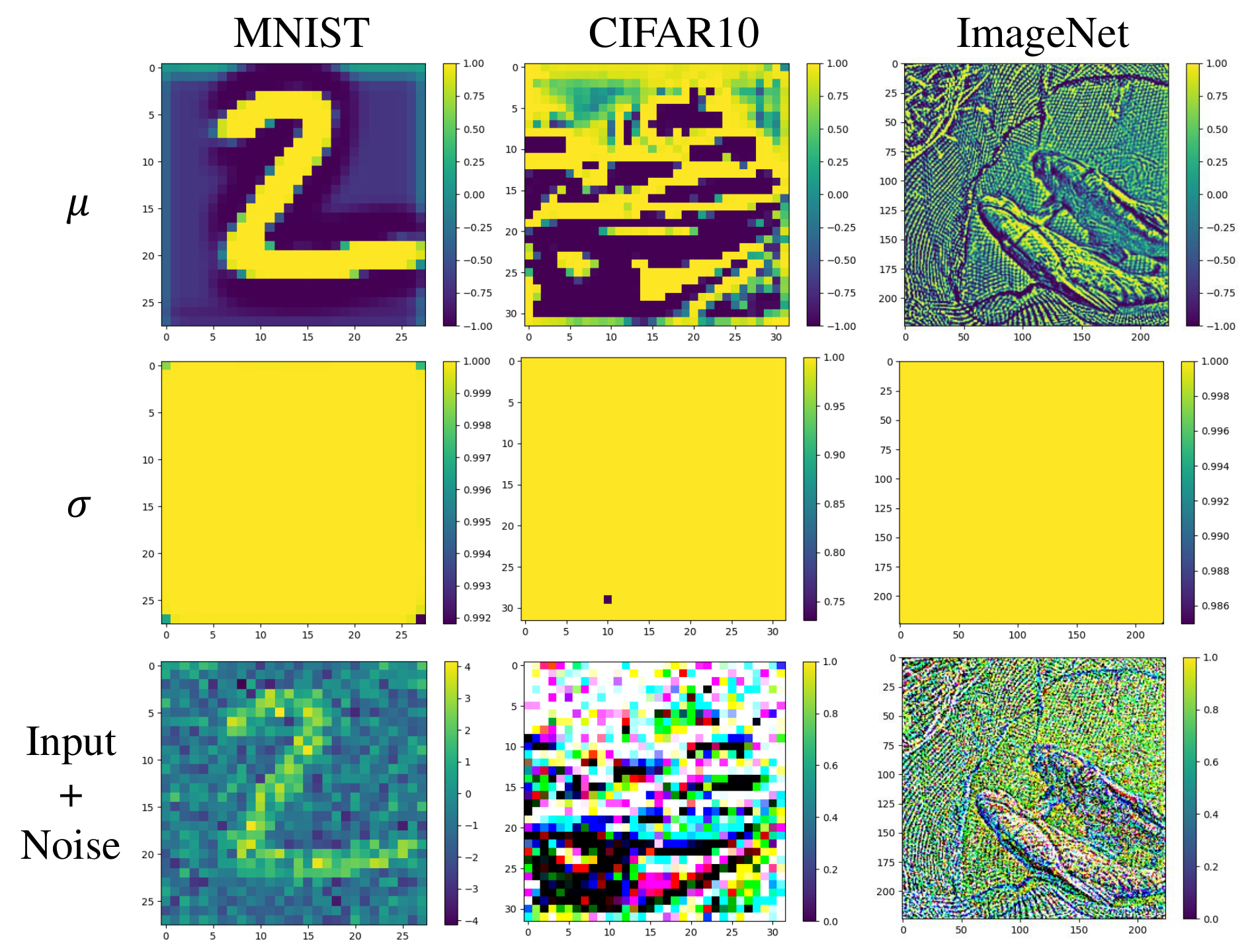}
    \caption{Examples of certification-wise anisotropic noise}
    \label{fig: visual personalized}
\end{subfigure}
\caption{Visualization of anisotropic vs. isotropic noise, based on the input in (a).}\vspace{-0.2in}
\label{fig:visualization}
\end{figure}

We present several examples of anisotropic (generated with ALM loss) and isotropic noise in Figure \ref{fig:visualization}. All the proposed anisotropic noise generation methods find better spatial distribution to generate the anisotropic noise. Both the pattern-wise and dataset-wise anisotropic reduce the variance on the key area, except the dataset-wise anisotropic noise on ImageNet (it seems the parameter generator does not find a constant key area on ImageNet due to the complicated data distribution in ImageNet). We also observe that the parameter generator for certification-wise anisotropic noise generates large mean offsets to compensate for the high $\sigma$ values. It turns out that with high variance ($\sigma_i \approx 1$), the object is still recognizable in the certification-wise anisotropic noise.

\subsection{Efficiency}
\label{apd: efficiency}
% \vspace{-0.1 in}
UCAN is a universal framework that can be readily integrated into existing randomized smoothing to boost performance. Whether the extra neural network components (parameter generator) in UCAN will degrade the efficiency of existing randomized smoothing is an important question. We show that the running time overhead resulting from the parameter generator is negligible compared to the running time of the certification, since for each input, the classifier needs to evaluate $N$ noise samples while $\mu$ and $\sigma$ are generated once. Typically, $N=100,000$. UCAN can be trained offline and tested online to boost the performance of randomized smoothing. We evaluate the online certification running time for certification-wise anisotropic noise generation and traditional randomized smoothing \cite{cohen2019certified} on ImageNet with four Tesla V100 GPUs and $2,000$ batch size, the average runtimes over $500$ samples are $27.43$s and $27.09$s per sample for our method and \cite{cohen2019certified}'s method, respectively. Thus, the NPG will only slightly increase the overall runtime. Also, Training a certification-wise NPG requires 200 epochs, taking 31.67 minutes on an H100 GPU.

%\subsection{Visualization and Efficiency}
%\label{sec:exp visualization}

%Please see the visualization and efficiency results in Appendix \ref{apd: visualization} and \ref{apd: efficiency}.

\section{Related Work}
\label{sec: RelatedWork}

In this section, we review the related works for certified defenses against evasion attacks on machine learning models.

\vspace{0.05in}

\noindent\textbf{Certified Defenses}. Certified defenses guarantee robustness against adversarial perturbations within a specified boundary (e.g., $\ell_1$, $\ell_2$, or $\ell_\infty$ ball of radius $R$). Existing approaches fall into two categories: exact and conservative certified defenses.
Exact methods leverage satisfiability modulo theories~\cite{katz2017reluplex,carlini2017provably,ehlers2017formal,huang2017safety} or mixed-integer linear programming~\cite{cheng2017maximum,lomuscio2017approach,fischetti2018deep,bunel2018unified} to precisely determine the existence of adversarial examples within $R$, but are not scalable to large networks.
Conservative methods use global/local Lipschitz constants~\cite{gouk2021regularisation,tsuzuku2018lipschitz,anil2019sorting,DBLP:conf/icml/CisseBGDU17,DBLP:conf/nips/HeinA17}, optimization-based certification~\cite{wong2018provable,wong2018scaling,raghunathan2018certified,dvijotham2018dual}, or layer-by-layer approaches~\cite{mirman2018differentiable,singh2018fast,DBLP:journals/corr/abs-1810-12715,DBLP:conf/icml/WengZCSHDBD18,DBLP:conf/nips/ZhangWCHD18}, which scale better but typically yield conservative or architecture-specific guarantees.
Neither can certify arbitrary classifiers, a gap addressed by randomized smoothing.

\vspace{0.05in}

\noindent\textbf{Randomized Smoothing}. Randomized smoothing was first explored by Lecuyer et al.~\cite{lecuyer2019certified}, who derived a loose theoretical robustness bound via Differential Privacy~\cite{dwork2006differential, dwork2008differential}. Cohen et al.~\cite{cohen2019certified} later provided the first tight guarantee, showing that adding Gaussian noise transforms any classifier into a smoothed classifier with certified $\ell_2$ robustness. Subsequent work extended smoothing to other $\ell_p$ norms: Teng et al.~\cite{teng2020ell} analyzed $\ell_1$ robustness with Laplace noise, and Lee et al.~\cite{DBLP:conf/nips/LeeYCJ19} studied $\ell_0$ robustness with uniform noise. Unified frameworks have also been proposed: Zhang et al.~\cite{zhang2020black} presented an optimization-based approach certifying $\ell_1$, $\ell_2$, and $\ell_\infty$ robustness; Yang et al.~\cite{yang2020randomized} introduced methods based on level sets and differentials to bound certified radii for various distributions and norms; Hong et al.~\cite{hong2022unicr} proposed a framework to approximately certify any $\ell_p$-norm adversary with any noise. However, existing randomized smoothing methods use fixed noise distributions (e.g., Gaussian, Laplace) applied uniformly to all input dimensions, overlooking data heterogeneity and limiting optimal protection for different inputs.

\vspace{0.05in}

\noindent\textbf{Data-Dependent Randomized Smoothing}. Data-dependent randomized smoothing improves certified robustness by optimizing noise distributions for individual inputs. Most existing approaches focus on isotropic noise, typically optimizing the variance to enlarge the certified radius. For instance, Alfarra et al.~\cite{alfarra2020data} optimize the Gaussian variance via gradient methods; Sukenik et al.~\cite{sukenik2021intriguing} model variance as an input-dependent function; Wang et al.~\cite{wang2020pretrain} use grid search for variance selection. However, these methods still inject noise with identical variance across all dimensions, due to the absence of anisotropic theory.

% Moreover, \cite{eiras2021ancer} ignores the mean-deviation that could occur in the transformation and find parameters of anisotropic noise via optimization.

\vspace{0.05in}

\noindent\textbf{Asymmetric Randomized Smoothing}. 
Prior work has already shown that randomized smoothing can certify \emph{anisotropic} regions: ANCER \cite{eiras2021ancer} introduces axis-aligned ellipsoidal certificates by scaling per-dimension noise, and RANCER \cite{rumezhak2023rancer} generalizes this to full (rotated) covariance structures. Both frameworks derive guarantees by enforcing Lipschitz constraints on the smoothed classifier and by (optionally) adapting noise parameters per input to enlarge the certified set.

Our approach departs from the Lipschitz-based route and instead relies on an explicit \emph{linear transformation} between isotropic and anisotropic noise distributions. Concretely, let the isotropic guarantee certify $x$ within radius $R$ under noise $\epsilon$. For any invertible covariance matrix $\Sigma$ (cf.\ our notation in Eq.~(4)), the same argument transfers to the anisotropic setting by the change of variables $\delta' = \Sigma \delta$: $\|\delta\|_p \leq R \;\Longleftrightarrow\; \|\Sigma^{-1} \delta'\|_p \leq R$, so the certified region becomes an ellipsoid (or generalized ellipsoid) parameterized by $\Sigma$. This algebraic reduction preserves the original Neyman--Pearson optimality of isotropic smoothing and avoids the typically looser (or at least not tighter) bounds introduced by gradient/Lipschitz upper bounds used in \cite{eiras2021ancer,rumezhak2023rancer}.

\cite{eiras2021ancer,rumezhak2023rancer} allow input-dependent noise but must cache those parameters to apply the same distribution for all evaluations during certification; otherwise the proof assumptions change. Our ``certification-wise'' design fixes the noise after observing the clean input and reuses it throughout the procedure, eliminating the memory-based workaround while retaining input adaptivity at $x$. By replacing Lipschitz bounds with a linear change of variables, we (i) simplify the theoretical pipeline, (ii) obtain tighter certificates (no gradient over-approximation), and (iii) remove the need for parameter memorization. Empirically (Table~\ref{tab: comparison CIFAR10}), at radius 1.25, our method certifies 75\% accuracy, far exceeding ANCER (31\%) and RANCER (1\%). Across all large radii, we achieve up to 182.6\% improvement over these baselines.

\section{Conclusion}
\label{sec: Conclusion}
% In this paper, we propose a novel randomized smoothing framework called UCAN. It can universally transform any randomized smoothing scheme with isotropic noise (for certified robustness) into randomized smoothing with anisotropic noise with strict robustness guarantees. Moreover, we have proposed a unified noise customization framework and three example methods to fine-tune anisotropic noise for classifier smoothing. Extensive experimental results on three datasets (MNIST, CIFAR10 and ImageNet) have validated that UCAN significantly boosts the certified robustness of existing randomized smoothing methods.

% In this paper, we introduce UCAN (Universally Certifies adversarial robustness with Anisotropic Noise), a novel randomized smoothing framework that transforms any isotropic noise-based scheme into schemes utilizing anisotropic noise, providing strict robustness guarantees. We also propose a unified noise customization framework with three methods for fine-tuning anisotropic noise in classifier smoothing. Extensive evaluation of the proposed method on MNIST, CIFAR10, and ImageNet confirms that UCAN substantially enhances the certified robustness of existing randomized smoothing methods.

In this paper, we introduce UCAN (Universally Certifies adversarial robustness with Anisotropic Noise), a novel and flexible randomized smoothing framework that transforms any isotropic noise-based scheme into schemes utilizing anisotropic noise, providing strict and theoretically grounded robustness guarantees. We also propose a unified and customizable noise customization framework with three methods for fine-tuning anisotropic noise in classifier smoothing. Extensive and comprehensive evaluation of the proposed method on MNIST, CIFAR10, and ImageNet confirms that UCAN substantially enhances the certified robustness of existing randomized smoothing methods. By enabling flexible and dimension-aware noise design, UCAN opens new possibilities for certified defense in more complex and heterogeneous data domains.

% (both the certified accuracy and the certified region).
% , which are facilitated by three proposed anisotropic noise generation methods. Finally, we have conducted extensive experiments on MNIST, CIFA10, and ImageNet datasets to validate the performance of UCAN by benchmarking with the state-of-the-art methods. 

\clearpage
\bibliographystyle{splncs04}
{\footnotesize
\bibliography{ARS}
}

% \clearpage
\appendices
% \section{Appendix}
{

\section{Proof of Theorem \ref{thm:our thm}}
\label{apd:proof}
We restate Theorem \ref{thm:our thm}:

\thmone*

\begin{proof}
Let $X = x + \epsilon$, where $\epsilon \in \mathbb{R}^d$ follows an isotropic noise distribution. Let $Y = x + \Sigma \epsilon + \mu$, where $\Sigma \in \mathbb{R}^{d \times d}$ is an invertible covariance matrix, and $\mu \in \mathbb{R}^d$ is a mean offset vector.

Given the input $x$, covariance matrix $\Sigma$, and mean offset $\mu$, define:
\begin{equation}
\label{eq:mu_prime}
\mu' = \mu + x - \Sigma x
\end{equation}
Then, the anisotropic input can be written as:
\begin{equation}
Y = \Sigma (x + \epsilon) + \mu' = \Sigma X + \mu'
\end{equation}

Define a transformation $h: \mathbb{R}^d \to \mathbb{R}^d$ as:
\begin{equation}
\label{eq:transform_h}
h(z) = \Sigma z + \mu'
\end{equation}
This transformation maps the isotropic input $z$ to the anisotropic space.

Given any deterministic or randomized function $f : \mathbb{R}^d \to \mathcal{C}$, define the anisotropic smoothed classifier as:
\begin{equation}
\label{eq:anisotropic_smoothed_classifier}
g'(x) = \arg \max_{c \in \mathcal{C}} \mathbb{P}\left( f\left( x + \Sigma \epsilon + \mu \right) = c \right)
\end{equation}

Suppose that for the anisotropic random variable $Y$, there exist $c_A' \in \mathcal{C}$ and $\underline{p_A}', \overline{p_B}' \in [0,1]$ such that:
\begin{equation}
\label{eq:anisotropic_condition}
\mathbb{P}\left( f(Y) = c_A' \right) \geq \underline{p_A}' \geq \overline{p_B}' \geq \max_{c \neq c_A'} \mathbb{P}\left( f(Y) = c \right)
\end{equation}

By the transformation $h$, the condition in Equation~\eqref{eq:anisotropic_condition} is equivalent to:
\begin{align}
\mathbb{P}\left( f(Y) = c_A' \right) &= \mathbb{P}\left( f\left( \Sigma X + \mu' \right) = c_A' \right) \\
&= \mathbb{P}\left( f\left( h(X) \right) = c_A' \right) \label{eq:isotropic_class_prob} \\
&\geq \underline{p_A}' \geq \overline{p_B}' \geq \max_{c \neq c_A'} \mathbb{P}\left( f\left( h(X) \right) = c \right) \label{eq:isotropic_class_prob2}
\end{align}

Consider a new classifier $f': \mathbb{R}^d \to \mathcal{C}$ defined as:
\begin{equation}
f'(z) = f\left( h(z) \right)
\end{equation}
This classifier maps the isotropic input $z$ to the class space $\mathcal{C}$ through the transformation $h$.

From Equations~\eqref{eq:isotropic_class_prob} and \eqref{eq:isotropic_class_prob2}, we have:
\begin{equation}
\label{eq:isotropic_condition}
\mathbb{P}\left( f'(X) = c_A' \right) \geq \underline{p_A}' \geq \overline{p_B}' \geq \max_{c \neq c_A'} \mathbb{P}\left( f'(X) = c \right)
\end{equation}
This is the prerequisite condition in the standard isotropic randomized smoothing theory (e.g., Theorem~\ref{thm:cohen's thm}).

Therefore, we can obtain the robustness guarantee with isotropic certified radius $R$ such that:
\begin{align}
\arg \max_{c \in \mathcal{C}} \mathbb{P}\left( f'(X + \delta) = c \right) = c_A' \label{eq:guarantee1} \\
\text{subject to } \left\| \delta \right\|_p \leq R\left( \underline{p_A}', \overline{p_B}' \right) \label{eq:guarantee2}
\end{align}

By the definition of $h$ in Equation~\eqref{eq:transform_h} and the fact that $Y = h(X)$, Equation~\eqref{eq:guarantee1} is equivalent to:
\begin{align}
\arg \max_{c \in \mathcal{C}} &\mathbb{P}\left( f'\left( X + \delta \right) = c \right) \\
&= \arg \max_{c \in \mathcal{C}} \mathbb{P}\left( f\left( h\left( X + \delta \right) \right) = c \right) \\
&= \arg \max_{c \in \mathcal{C}} \mathbb{P}\left( f\left( \Sigma \left( X + \delta \right) + \mu' \right) = c \right) \\
&= \arg \max_{c \in \mathcal{C}} \mathbb{P}\left( f\left( \Sigma X + \Sigma \delta + \mu' \right) = c \right) \\
&= \arg \max_{c \in \mathcal{C}} \mathbb{P}\left( f\left( Y + \Sigma \delta \right) = c \right) = c_A' \label{eq:guarantee3}
\end{align}

Define $\delta' = \Sigma \delta$, with $\delta = \Sigma^{-1} \delta'$ since $\Sigma$ is invertible.

Substituting $\delta = \Sigma^{-1} \delta'$ into the condition $\left\| \delta \right\|_p \leq R\left( \underline{p_A}', \overline{p_B}' \right)$, we obtain:
\begin{equation}
\left\| \Sigma^{-1} \delta' \right\|_p \leq R\left( \underline{p_A}', \overline{p_B}' \right)
\end{equation}

Therefore, the guarantee in Equations~\eqref{eq:guarantee1} and \eqref{eq:guarantee2} is equivalent to:
\begin{align}
\arg \max_{c \in \mathcal{C}} \mathbb{P}\left( f\left( Y + \delta' \right) = c \right) = c_A' \\
\text{subject to } \left\| \Sigma^{-1} \delta' \right\|_p \leq R\left( \underline{p_A}', \overline{p_B}' \right)
\end{align}

By Equation~\eqref{eq:anisotropic_smoothed_classifier}, we have:
\begin{equation}
g'\left( x + \delta' \right) = \arg \max_{c \in \mathcal{C}} \mathbb{P}\left( f\left( Y + \delta' \right) = c \right) = c_A'
\end{equation}
for all $\delta' \in \mathbb{R}^d$ satisfying $\left\| \Sigma^{-1} \delta' \right\|_p \leq R\left( \underline{p_A}', \overline{p_B}' \right)$.

This completes the proof.
\end{proof}

\section*{Binary case of Theorem \ref{thm:our thm}}
\label{apd:binary}
\begin{restatable}[\textbf{Universal Transformation for Anisotropic Noise}]{theorem}{thmbinary}
\label{thm:binary}
Let $f \ : \ \mathbb{R}^d \to \mathcal{C}$ be any deterministic or random function. Suppose that for the isotropic input $X$ in Theorem \ref{thm:cohen's thm}, the certified radius function for the binary case is $R(\cdot)$. Then, for the corresponding anisotropic input $Y$ such that $Y=x+\epsilon^\top\sigma+\mu$, if there exist $c'_A\in \mathcal{C}$ and $\underline{p_A}' \in (1/2,1]$ such that:

\begin{equation}\small
    \mathbb{P}(f(Y)=c'_A) \geq \underline{p_A}' \geq \frac{1}{2}
\end{equation}

Then $g'(x+\delta')= \arg \max_{c\in \mathcal{C}} \mathbb{P}(f(Y)=c) =c'_A $ for all $||\delta' \oslash \sigma||_p \leq R(\underline{p_A}')$
where $g'$ denotes the anisotropic smoothed classifier, $\delta'$ denotes the perturbation injected to $g'$.

\begin{proof} The proof for the binary case is similar to the proofs for the multiclass-case (see Appendix \ref{apd:proof}).  
\end{proof}

\end{restatable}

\section{Additional Metric for Certified Region (Binary-case)}
\label{apd: metric}

\begin{table*}[!t]

    \centering
    \caption{ALM (binary-case) for randomized smoothing with independent anisotropic noise. $d$ is the dimension size. $\Phi^{-1}$ is the inverse CDF of Gaussian distribution. $\lambda$ is the scalar parameter of the isotropic noise. $\sigma$ is the anisotropic scale multiplier.}
   \vspace{-0.1in}
    \resizebox{0.8\linewidth}{!}{
    \begin{tabular}{c c c c c}
    \hline
    \hline
         Distribution & PDF & Adv.  &Anisotropic Guarantee &   ALM \\
         % & Alt. Lebesgue Measure \\
    \hline
         Gaussian \cite{cohen2019certified} & $\propto e^{-\|\frac{z}{\lambda}\|_2^2}$ & $\ell_2$ & $||\delta \oslash \sigma||_2 \leq\lambda (\Phi^{-1}(\underline{p_A}'))$ &  $ \sqrt[d]{\prod_{i=1}^d\sigma_i} \lambda (\Phi^{-1}(\underline{p_A}'))$\\
         Gaussian \cite{yang2020randomized} & $\propto e^{-\|{\frac{z}{\lambda}}\|_2^2}$ & $\ell_1$ &  $||\delta \oslash \sigma||_1 \leq \lambda (\Phi^{-1}(\underline{p_A}'))$ 
         & $\sqrt[d]{\prod_{i=1}^d\sigma_i} \lambda (\Phi^{-1}(\underline{p_A}'))$\\
         &  & $\ell_\infty$ & $||\delta \oslash \sigma||_\infty \leq \lambda (\Phi^{-1}(\underline{p_A}'))/\sqrt{d}$ 
         & $\sqrt[d]{\prod_{i=1}^d\sigma_i} \lambda (\Phi^{-1}(\underline{p_A}'))/\sqrt{d}$ \\
         Laplace \cite{teng2020ell} & $\propto e^{-\|\frac{z}{\lambda}\|_1}$ & $\ell_1$ & $||\delta \oslash \sigma||_1 \leq- \lambda \log(2(1-\underline{p_A}'))$ 
         & $\sqrt[d]{\prod_{i=1}^d\sigma_i} \lambda \log(2(1-\underline{p_A}'))$\\ 
         Exp.  $\ell_\infty$ \cite{yang2020randomized} & $\propto e^{-\|\frac{z}{\lambda}\|_\infty}$ & $\ell_1$ & $||\delta \oslash \sigma||_1 \leq 2d \lambda(\underline{p_A}'-\frac{1}{2})$ 
         & $2\sqrt[d]{\prod_{i=1}^d\sigma_i}d\lambda(\underline{p_A}'-\frac{1}{2})$\\
          &  & $\ell_\infty$ & $||\delta \oslash \sigma||_\infty \leq \lambda \log(\frac{1}{2(1-\underline{p_A}')})$ 
          & $ \sqrt[d]{\prod_{i=1}^d\sigma_i}\lambda \log(\frac{1}{2(1-\underline{p_A}')})$\\
         Uniform $\ell_\infty$ \cite{lee2018simple}& $\propto \mathbb{I}(\|z\|_\infty\leq\lambda)$ & $\ell_1$ & $ ||\delta \oslash \sigma||_1 \leq 2 \lambda (\underline{p_A}'-\frac{1}{2})$ 
         & $2 \sqrt[d]{\prod_{i=1}^d\sigma_i}\lambda  (\underline{p_A}'-\frac{1}{2})$\\
         & & $\ell_\infty$ & $||\delta \oslash \sigma||_\infty \leq 2\lambda(1-\sqrt[d]{\frac{3}{2}-\underline{p_A}'})$ 
         & $2 \sqrt[d]{\prod_{i=1}^d\sigma_i}\lambda(1-\sqrt[d]{\frac{3}{2}-\underline{p_A}'})$\\
         Power Law $\ell_\infty$ \cite{yang2020randomized} & $\propto \frac{1}{(1+\|\frac{z}{\lambda}\|_\infty)^a}$ & $\ell_1$ & $||\delta \oslash \sigma||_1 \leq \frac{2d\lambda}{a-d}(\underline{p_A}-\frac{1}{2})$ 
         & $2\sqrt[d]{\prod_{i=1}^d\sigma_i}\frac{d\lambda}{a-d}(\underline{p_A}-\frac{1}{2})$\\
    \hline
    \hline
    \end{tabular}}
     \label{tab:ALM formulas}
     \vspace{-0.05in}
\end{table*}

How to develop a general metric for evaluating the robustness region for randomized smoothing with anisotropic noise is an important but challenging problem. We observe that the guarantee in Theorem \ref{thm:our thm} forms a certified region, within which the perturbation is certifiably safe to the smoothed classifier. The $\ell_p$-norm bounding on the scaled perturbation $\delta'\oslash\sigma$ results in the anisotropy of the certified region around the input. We illustrate the asymmetric certified region for different $\ell_p$-norm guarantees in Figure \ref{fig:anisotropic_vs_isotropic}. It shows that if the $\delta$ space is a 2-dimension space, then the guarantee of asymmetric RS draws a rhombus, ellipse, and rectangle in $\ell_1$, $\ell_2$, and $\ell_\infty$ norms, respectively. Within the asymmetric region, we can find an isotropic region that also satisfies the robustness guarantee (a subset of the anisotropic region), which represents an explicit certified radius as depicted in Corollary \ref{corollary}.

However, evaluating the performance of randomized smoothing with anisotropic noise via Eq. (\ref{eq: worse-case guarantee}) in Corollary \ref{corollary}, although formally guaranteed with robustness, fails to capture the full certified regions. 

Specifically, Eq. (\ref{eq: worse-case guarantee}) evaluates the performance only based on the blue region in Figure \ref{fig:anisotropic_vs_isotropic}, but the Eq. (\ref{eq:guarantee_general}) in Theorem \ref{thm:our thm} actually guarantees that all the $\delta$ (perturbations) within the green region do not change the perturbation. Therefore, to fairly and accurately evaluate the performance of certification via anisotropic noise, we need to develop an auxiliary metric that can cover the entire certified region in highly-dimensional $\delta$ space as a complement besides the certified radius. 

From another perspective, evaluating the performance of randomized smoothing can be considered as evaluating the size of the robust perturbation set $S(n,p)$.

Consider the Euclidean structure, $S(d, p)$ is a finite set in $d$-dimensional Euclidean space. Therefore, we leverage the Lebesgue measure \cite{bartle2014elements} to compute the size of $S(d,p)$ (see Theorem \ref{thm: Lebesgue measure}).

We observe that for a fixed $d$ and $p$, the $\frac{(2\Gamma(1+\frac{1}{p}))^d}{\Gamma(1+\frac{d}{p})}$ factor in the Lebesgue measure is a constant. Then, when comparing the Lebesgue measure in the same norm $\ell_p$ and the same space $\mathbb{R}^d$, the constant term can be ignored. Also, the $R^d$ factor can lead to infinite numeral computation, thus we also scale the Lebesgue measure by calculating the $d$-th root. As a result, we define the Alternative Lebesgue Measure (ALM) of the robust perturbation set with the same $d$ and $p$ as:

{\footnotesize
\begin{equation}
    V_S'=\sqrt[d]{\prod_{i=1}^d\sigma_i}R
\end{equation}}

Table \ref{tab:ALM formulas} presents the ALM formulas of common RS methods.

\noindent\textbf{Alternative Lebesgue Measure\footnote{ALM is like the normalized radius of the certified region in all $d$ dimensions.} vs Radius}. When the multipliers of the scale parameter for anisotropic noise $\sigma_1=\sigma_2=...=1$, the noise turns into the isotropic noise and the alternative Lebesgue measure turns into the certified radius $R$. Therefore, the alternative Lebesgue measure can be treated as a generalized metric compared to the certified radius. %which considers the anisotropy of the robustness region. 
This generalization based on the certified radius also enables us to fairly compare the randomized smoothing based on anisotropic noise with isotropic noise.

Note that the new metric ALM is not developed to bound the perturbation, but to accurately measure the certified guarantees of randomized smoothing with anisotropic noise.

\begin{figure*}[!t]
    \centering
    \resizebox{0.8\textwidth}{!}{ % 这里开始缩放
    \begin{minipage}{\textwidth}
    \begin{subfigure}[b]{0.33\linewidth}
        \centering
        \includegraphics[width=\linewidth]{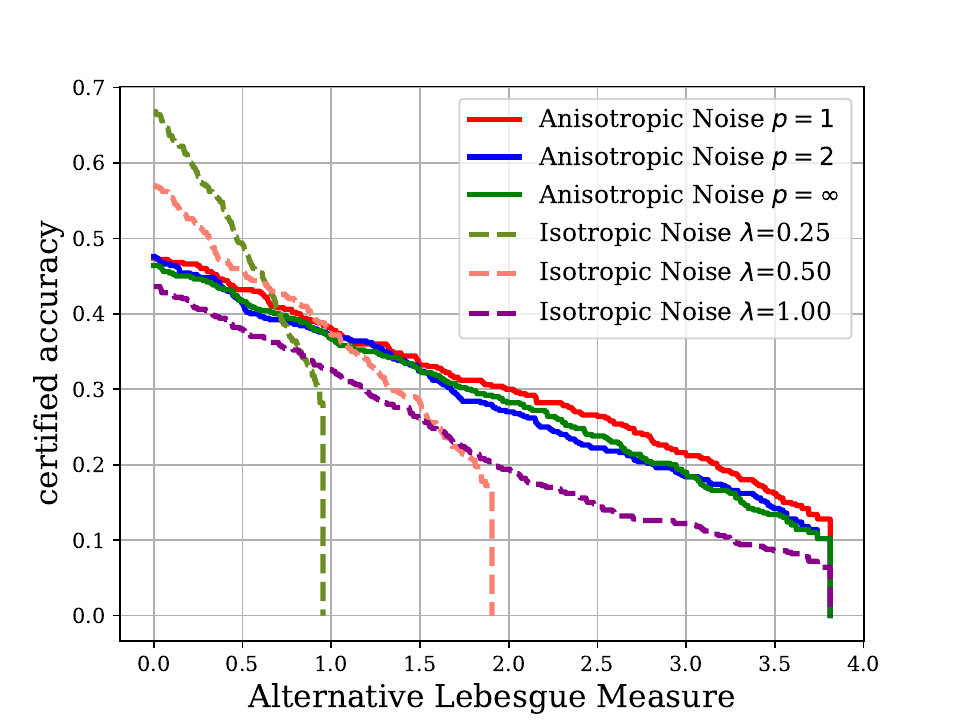}
        \caption{Pattern-wise (ImageNet)}
        \label{fig:exp pre-assigned ImageNet}
    \end{subfigure}
    \hspace{-0.1in}
    \begin{subfigure}[b]{0.33\linewidth}
        \centering
        \includegraphics[width=\linewidth]{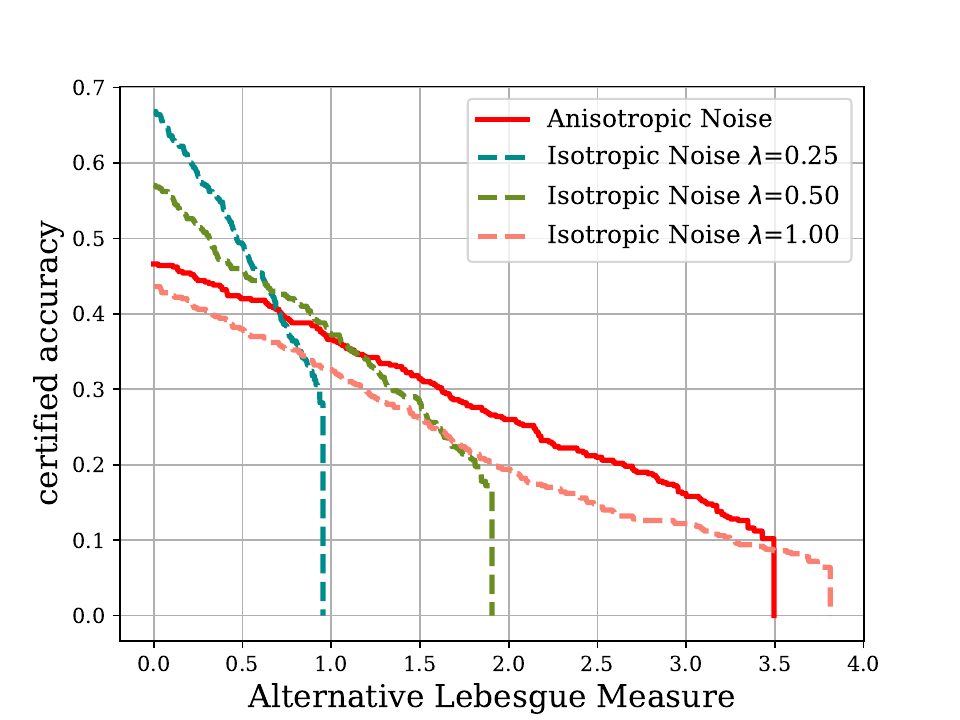}
        \caption{Dataset-wise (ImageNet)}
        \label{fig:exp universal ImageNet}
    \end{subfigure}
    \hspace{-0.1in}
    \begin{subfigure}[b]{0.33\linewidth}
        \centering
        \includegraphics[width=\linewidth]{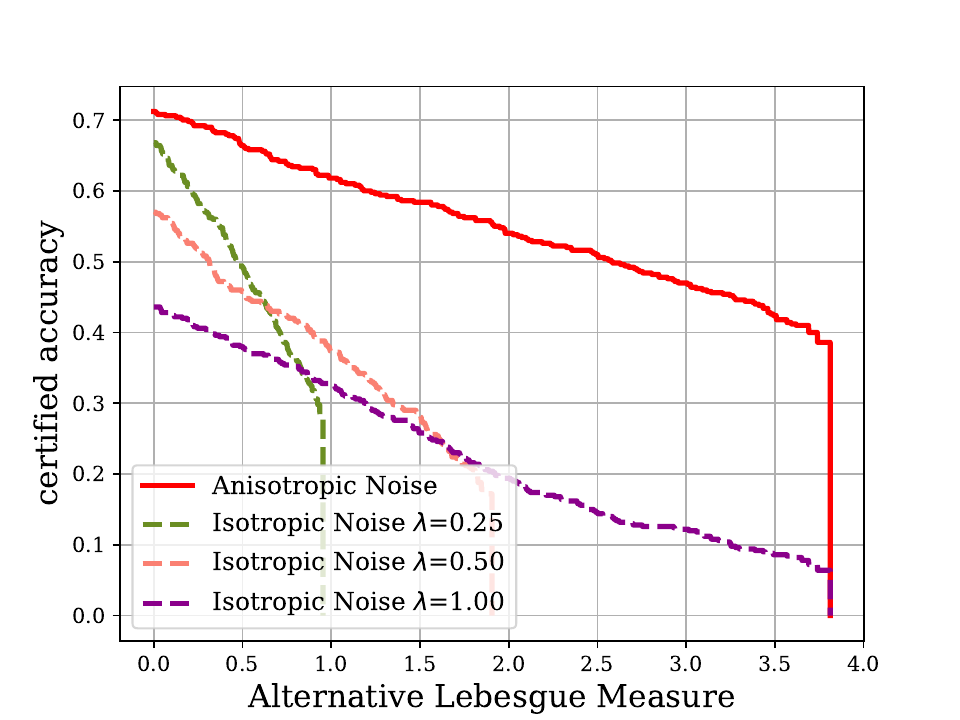}
        \caption{Certification-wise (ImageNet)}
        \label{fig:exp personalized ImageNet}
    \end{subfigure}
    \end{minipage}
    } % 这里结束缩放
    \caption{Comparison of randomized smoothing performance on ImageNet with anisotropic noise and that with isotropic noise (Gaussian distribution for certified defense against $\ell_2$ perturbations, comparing with \cite{cohen2019certified}).}\vspace{-0.1in}
    \label{fig:exp three methods imagenet}
\end{figure*}

\section{Proof of Theorem \ref{thm: Lebesgue measure}}
\label{apd: proof thm 5.4}

\begin{proof}    
\begin{definition}[\textbf{d-dimensional Generalized Super-ellipsoid}]
\label{def:super-ellipsoid}
The d-dimensional generalized super-ellipsoid ball is defined as
{\footnotesize
\begin{equation}
    E(d,p)=\{(\delta_1,\delta_2,...,\delta_d):\sum_{i=1}^d|\frac{\delta_i}{c_i}|^{p_i}\leq 1, p_i>0\}
\end{equation}}
\end{definition}

\begin{definition}[\textbf{Euler Gamma Function}]
\label{def: euler gamma}

The Euler gamma function is defined by
{\footnotesize
\begin{equation}
    \Gamma(\beta)=\int_0^\infty \alpha^{\beta-1}e^{-\alpha}d\alpha
\end{equation}}
There are some properties of $\Gamma$: 1) For all $\beta>0$, $\beta\Gamma(\beta)=\Gamma(\beta+1)$, 2) For all positive integers $n$, $\Gamma(n)=(n-1)!$, and 3) $\Gamma(1/2)=\sqrt{\pi}$. 
\end{definition}

\begin{lemma}[\textbf{Lebesgue Measure of Generalized Super-ellipsoids \cite{ahmed2018volumes}}]
\label{Lemma:volumn equation}
The Lebesgue measure of the generalized super-ellipsoids defined in Definition \ref{def:super-ellipsoid} is given by
{\footnotesize
\begin{equation}
\label{eq：apd lebesgue}
    V_E(d,p)=2^d \frac{\prod_{i=1}^d c_i \Gamma(1+\frac{1}{p_i})}{\Gamma(1+\sum_{i=1}^{d}\frac{1}{p_i})}; p_i>0, d=1,2,3,...
\end{equation}}
\end{lemma}

We leverage Lemma \ref{Lemma:volumn equation} to prove Theorem \ref{thm:our thm}. Let $p_i=p$, the d-dimensional robust perturbation set is equivalent to 
{\footnotesize
\begin{equation}
S(d,p)=\{(\delta_1,\delta_2,...,\delta_d): \sum_{i=1}^d |\frac{\delta_i}{\sigma_iR}|^p\leq 1 \}   
\end{equation}
}
The Lebesgue measure in Lemma \ref{Lemma:volumn equation} will be

{\footnotesize
\begin{align}
    V_E(d,p)=&2^d \frac{\prod_{i=1}^d c_i R \Gamma(1+\frac{1}{p})}{\Gamma(1+\sum_{i=1}^{d}\frac{1}{p})}=\frac{(2R\Gamma(1+\frac{1}{p}))^d\prod_{i=1}^d \sigma_i}{\Gamma(1+\frac{d}{p})}; \nonumber\\
    &p>0, d=1,2,3,...
\end{align}
}
Thus, this completes the proof. 
\end{proof}

\begin{figure}[!t]

\centering
\begin{subfigure}[b]{0.51\linewidth}
    \centering
    \includegraphics[width=\linewidth]{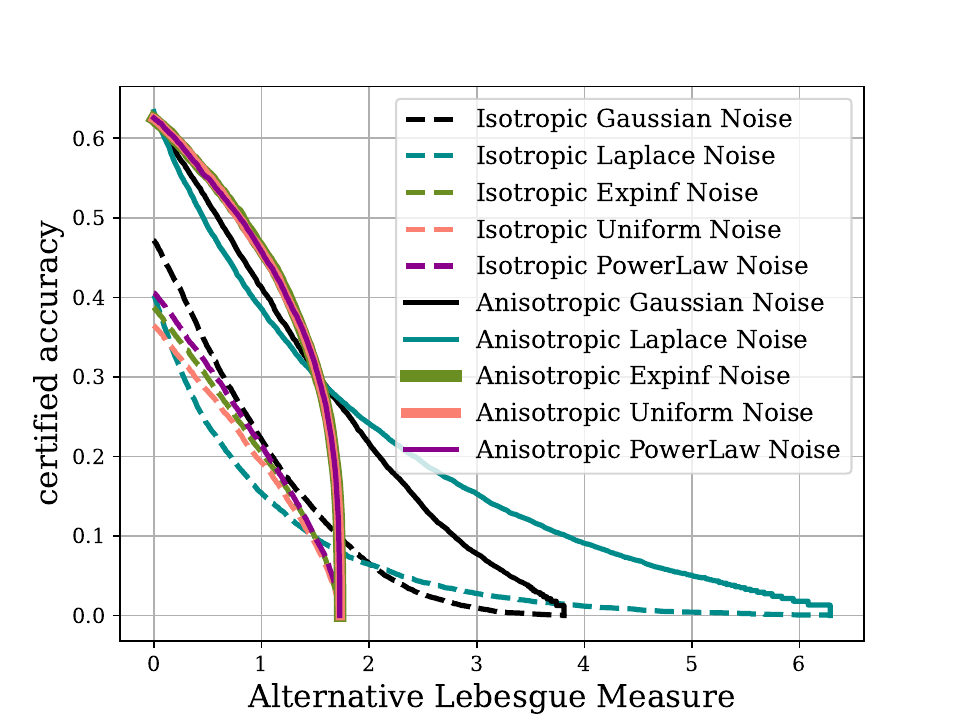}
    \caption{Pattern-wise vs. $\ell_1$ pert.}
    \label{fig:exp universality pre-assgiend l1}
\end{subfigure}
\hspace{-0.2in} % Adjust or remove spacing based on your layout needs
\begin{subfigure}[b]{0.51\linewidth}
    \centering
    \includegraphics[width=\linewidth]{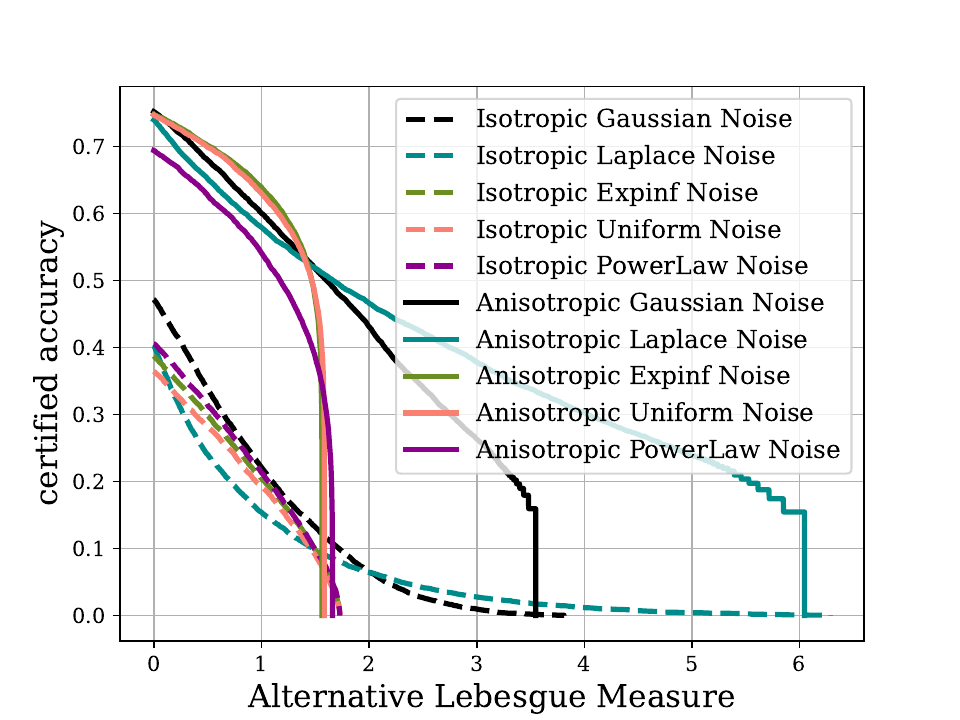}
    \caption{Dataset-wise vs. $\ell_1$ pert.}
    \label{fig:exp universality universal l1}
\end{subfigure}

\begin{subfigure}[b]{0.51\linewidth}
    \centering
    \includegraphics[width=\linewidth]{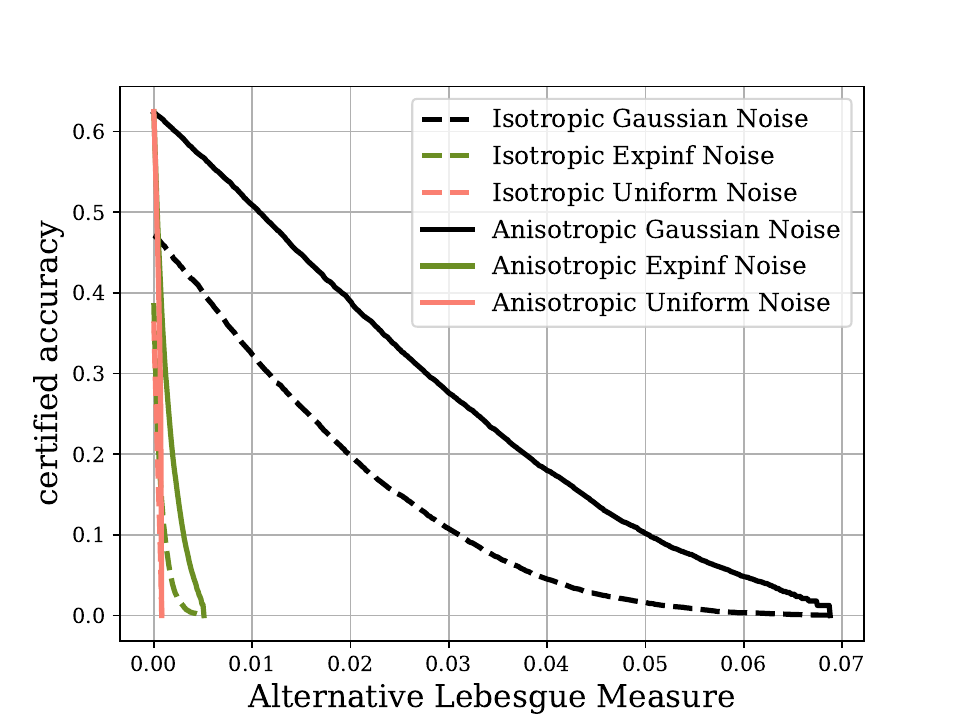}
    \caption{Pattern-wise vs. $\ell_\infty$ pert.}
    \label{fig:exp universality pre-assigned linf}
\end{subfigure}
\hspace{-0.2in} % Adjust or remove spacing based on your layout needs
\begin{subfigure}[b]{0.51\linewidth}
    \centering
    \includegraphics[width=\linewidth]{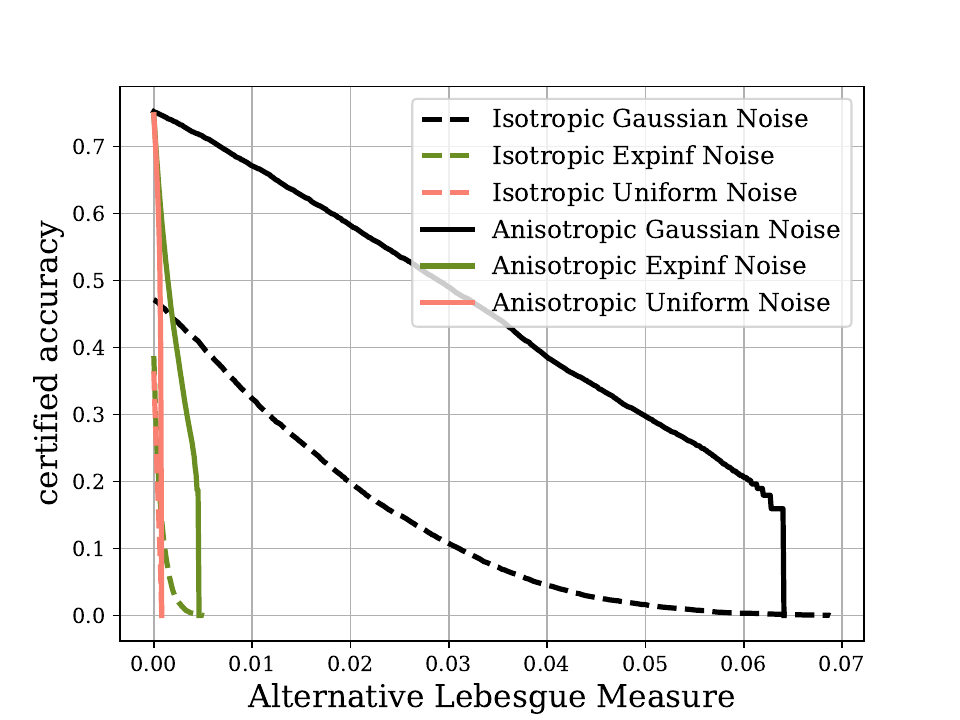}
    \caption{Dataset-wise vs. $\ell_\infty$ pert.}
    \label{fig:exp universality universal linf}
\end{subfigure}

\caption{UCAN with pattern-wise and dataset-wise anisotropic noise vs. RS with isotropic noise -- different noise PDFs against different $\ell_p$ perturbations (universality) on CIFAR10.}
\label{fig:exp universality two methods}
%\vspace{-0.2in} % Adjust vertical spacing as needed
\end{figure}

\section{Additional Experiments}
\label{apd:more exps}

\subsection{Anisotropic noise vs. isotropic noise}

We also present the performance comparison of the anisotropic and isotropic RS on ImageNet in \autoref{fig:exp three methods imagenet}. It shows that similar to the results in the MNIST and CIFAR10, the pattern-wise and dataset-wise anisotropic noise can improve the performance of certified accuracy w.r.t. the ALM moderately, while the certification-wise anisotropic noise can significantly boost the performance with the best optimality.

\subsection{Universality of Pattern-wise and Dataset-wise Methods}

Similar to the certification-wise noise, pattern-wise and dataset-wise noise are also universal to different $\ell_p$ norms and different PDFs. As shown in \autoref{fig:exp universality two methods}, the pattern-wise and dataset-wise anisotropic noise can also significantly boost the performance of isotropic RS on various settings.

}

\end{document}